\documentclass[times, review, 10pt]{elsarticle}
\usepackage{amsfonts}
\usepackage{amsmath}
\usepackage{amssymb}
\usepackage{graphicx}
\usepackage{algorithm}
\usepackage{algorithmic}
\usepackage{cases}
\usepackage{geometry, tabularx}
\usepackage{multirow}
\usepackage{enumerate}
\usepackage{bm}
\usepackage{pifont}
\usepackage{footnote}
\usepackage{float}
\usepackage{titlesec}
\usepackage[titletoc]{appendix}
\usepackage{setspace}
\usepackage{makecell}
\usepackage{mathrsfs}

\providecommand{\U}[1]{\protect\rule{.1in}{.1in}}

\newtheorem{thm}{\bf Theorem}      
\newtheorem{lem}{\bf Lemma}

\newtheorem{rem}{\bf Remark}

\usepackage{hyperref}
\hypersetup{hypertex=true,
            colorlinks=true,
            linkcolor=blue,
            anchorcolor=green,
            citecolor=red}

\usepackage{graphicx,color}
\usepackage[bf,SL,BF]{subfigure}

\newenvironment{proof}[1][Proof]{\noindent\textbf{#1.} }{\ \rule{0.5em}{0.5em}}
\normalsize
\titlespacing*{\section} {0pt}{9pt}{0pt}
\numberwithin{equation}{section}
\allowdisplaybreaks

\usepackage{booktabs}
\begin{document}

\begin{frontmatter}

\title{Leveraging Noisy Manual Labels as Useful Information: An Information Fusion Approach for Enhanced Variable Selection in Penalized Logistic Regression}

\author{Xiaofei Wu}
\affiliation{organization={School of Mathematics and Statistics, Yunnan University},
            addressline={No.55 Daxuecheng South Rd., Shapingba District}, 
            city={Chongqing},
            postcode={401331}, 
            state={Chongqing},
            country={China}}

\author{Rongmei Liang}
\affiliation{organization={Department of Statistics and Data Science, Southern University of Science and Technology},
            addressline={No. 1088 Xueyuan Avenue, Nanshan District}, 
            city={Shenzhen},
            postcode={518055}, 
            state={Guangdong},
            country={China}}


\begin{abstract}
In large-scale supervised learning, penalized logistic regression (PLR) effectively mitigates overfitting through regularization, yet its performance critically depends on robust variable selection. This paper demonstrates that label noise introduced during manual annotation, often dismissed as a mere artifact, can serve as a valuable source of information to enhance variable selection in PLR. We theoretically show that such noise, intrinsically linked to classification difficulty, helps refine the estimation of non-zero coefficients compared to using only ground truth labels, effectively turning a common imperfection into a useful information resource. To efficiently leverage this form of information fusion in large-scale settings where data cannot be stored on a single machine, we propose a novel partition insensitive parallel algorithm based on the alternating direction method of multipliers (ADMM). Our method ensures that the solution remains invariant to how data is distributed across workers, a key property for reproducible and stable distributed learning, while guaranteeing global convergence at a sublinear rate. Extensive experiments on multiple large-scale datasets show that the proposed approach consistently outperforms conventional variable selection techniques in both estimation accuracy and classification performance, affirming the value of intentionally fusing noisy manual labels into the learning process. 
\end{abstract}

\begin{keyword}
Distributed computing \sep Information fusion \sep Logistic regression \sep Massive data \sep Noisy manual labels    \sep Variable selection   
\end{keyword} 
\end{frontmatter}
\section{Introduction}
Label noise is prevalent in classification problems within statistics \cite{Yao2021AsymmetricEC} and machine learning \cite{Bai2024ACS}. A common approach focuses on developing methods to maintain the efficiency of classifiers in the presence of label noise. 
Frenay and Verleysen \cite{Frnay2014ClassificationIT} provided overviews of some works on this problem, and categorized the methods into four types: label noise-robust methods, data cleaning methods, probabilistic label noise-tolerant  methods and model-based label noise-tolerant  methods.  However, all these methods treat noise as either harmful or a lack of statistical information.

Recently, Li \cite{Li2024PositiveN} pointed out that not all noise is harmful; certain types of noise can exhibit useful characteristics in specific tasks. Almost simultaneously, similar findings on classification problems emerged in statistical learning. For example, \cite{Cannings2018ClassificationWI} covered k-nearest neighbour, support vector machine and linear discriminant analysis, \cite{Ahfock2021HarmlessLN} dealt with logistic regression, and \cite{Lee2021BinaryCW} focused on the logistic regression and support vector machine. 
Inspired by these works, this paper delves into the potential of leveraging noisy information fusion in the context of large-scale (big) data. Specifically, it investigates whether there exists a type of useful noise within this information-fusion scenario that can improve the variable selection performance of PLR in \cite{Zou2006TheAL}.   Variable selection refers to the process of choosing a subset of the most important variables from a large set of candidate variables when building a statistical model, to improve predictive performance and interpretability. By excluding irrelevant or redundant variables, variable selection helps reduce the complexity of a logistic regression model, prevent overfitting, and enhance stability and interpretability.

The PLR loss function in \cite{Van2020StackedPL} is given by:
\begin{equation}
\begin{aligned}\label{pqr}
    L(\bm{\beta}) = -\sum_{i=1}^{n} \left[ y_i \log({p}(\bm x_i,\bm \beta)) + (1 - y_i) \log(1 - {p}(\bm x_i,\bm \beta)) \right] + \sum_{j=1}^{d} P_{\lambda}(|\beta_j|),
\end{aligned}    
\end{equation}
where
\begin{align}\label{p1}
{p}(\bm x_i, \bm \beta) = \frac{1}{1 + \exp({-\bm{\beta}^\top \bm{x}_i)}}
\end{align}
is the predicted probability for the \(i\)-th observation, \(y_i \in \{0, 1 \}\) is the observed label, \(\bm{x}_i\) is the feature vector, \(\bm{\beta}\) is the vector of regression coefficients, and \(P_{\lambda}(|\beta_j|)\) is the regularization term with parameter \(\lambda\). A popular regularization term with variable selection ability is LASSO (least absolute shrinkage and selection operator), also known as  $\ell_1$ regularization  ($\lambda \sum_{j=1}^{d} |\beta_j|$). However, many studies  (see \cite{Zou2006TheAL} and \cite{Hastie2015StatisticalLW}) have shown that under the assumption of sparse coefficients (where most coefficients are zero and only a few are non-zero), the LASSO  regularization term does not possess variable selection consistency. In other words, the estimator of LASSO-PLR in \eqref{pqr} cannot guarantee that the true non-zero coefficients will be estimated as non-zero.
To ensure the consistency of variable selection in PLR, an effective alternative is the adaptive (weighted) LASSO (ALASSO \cite{Zou2006TheAL}) or some nonconvex regularization terms in \cite{Hastie2015StatisticalLW}. Due to space limitations in this paper, ALASSO serves as the primary regularization term. 

In supervised learning, manual labeling of training samples is a common practice \cite{Frnay2014ClassificationIT}, especially when dealing with large datasets where some labels are missing or are not easily obtained. When the labeling task is inherently challenging, the provided labels may not always align with the true labels, thus introducing label noise into the training dataset (that is, label noise  comes from classification difficulty). Moreover, when multiple experts manually label the same instance, it is possible for them to assign different category labels, a scenario of label noise that is also considered in this study. To facilitate the use of statistical analysis tools in examining the impact of this noise, we adopt the strategy from \cite{Ahfock2021HarmlessLN}, where manual labels are randomly sampled based on the posterior probabilities of class membership. In this setting, the noise generated by the label is only related to the difficulty of classification. A more relaxed version involves using a random effects model (a.k.a.  finite-mixture model in \cite{Fraley2002ModelBasedCD} and \cite{McLachlan2000FiniteMM}) that generates manual labels based on imperfect approximations of the posterior probabilities of class membership.

The main contributions of this paper are summarized  as follows:
\begin{enumerate}
\item  The first contribution of this paper is to demonstrate that in large-scale PLR, the noise in manual labels (which is solely related to the difficulty of classification) can offer statistical information that is precisely identical to that of the truth label. In other word, fusing this noisy information from multiple sources is beneficial for variable selection in PLR.

Specifically,  when the manual label generation method is based on the posterior probabilities of two categories, the variable selection performance of PLR in the noisy scenario has the following three properties.
\begin{itemize}
    \item  The \textbf{variable selection} performance of PLR is the same as that  in the noise-free (truth labels) one. That is, PLR can accurately identify non-zero true variables regardless of whether there is noise or not.

    \item  In both scenarios (with and without noise), the non-zero variables identified by PLR exhibit a similar \textbf{asymptotic distribution}. Moreover, when the number of experts manually labeled is 1, these two asymptotic distributions will be the same.

    \item When estimating non-zero variables in  PLR with multiple-expert manual labels, the \textbf{relative efficiency} is greater than 1 compared to the  noise-free scenario. Put simply, it achieves a more accurate estimation of non-zero variables than the noise-free case.
Notably, this relative efficiency increases as the number of experts grows.  
\end{itemize}
When the manual label generation method is derived from mixed effects models (i.e. unclear understanding of posterior probabilities), PLR can also have the above three properties.

\item The second contribution of this paper is to develop a partition-insensitive parallel ADMM algorithm for solving large-scale PLR with manually-labeled data stored in a distributed manner across multiple machines. The most prominent feature of this distributed algorithm is its insensitivity to data partitioning. In distributed computing, it means that no matter how the data is stored, regardless of the number of machines the data is stored on and the amount of data each machine holds, the iterative solutions obtained by the algorithm stay consistent. This feature is of great significance for improving result predictability, system scalability, and computational efficiency.   Moreover, this solution is exactly the same as the one obtained when a single machine processes the entire dataset. In theory, the distributed algorithm can be proven to have a sublinear convergence rate.
\end{enumerate}

The paper is structured as follows. Section \ref{sec2} offers preliminaries and a literature review. Section \ref{sec3} presents the variable selection properties and asymptotic relative efficiency of PLR with beneficial noise in manual labels. Section \ref{sec4} proposes a new parallel algorithm for PLR under manual labels. Section \ref{sec5} conducts numerical experiments on real datasets. Section \ref{sec6} concludes the paper and suggests future research directions. Supplementary experiments and technical proofs are in the  supplementary material.  The R package for implementing the algorithm proposed in this paper and some numerical experiments are available at the following \url{https://github.com/xfwu1016/PLRBN}.
\section{Preliminaries and Literature Reviews}\label{sec2}
\subsection{Penalized Logistic Regression}\label{sec21}
For logistic regression models, the density function can be expressed as 
\begin{align}\label{plden}
  f(y \mid \bm x, \bm \beta ) = \left({p}(\bm x, \bm \beta)\right)^{y} \left(1-{p}(\bm x, \bm \beta)\right)^{1 - y},   
\end{align}
where ${p}(\cdot)$ is defined in \eqref{p1}. The Fisher information matrix is defined  as
\begin{align}\label{fis}
I(\boldsymbol{\beta}) = \mathbb  E_{\bm x} \left[ \bm{x} \bm{x}^\top {p}(\bm x, \bm \beta) (1 - {p}(\bm x, \bm \beta) )\right].
\end{align}
The dataset $\{\bm x_i, y_i \}_{i = 1}^n$ consists of independently-collected observation samples regarding $\{\bm x, y \}$. Taking the negative log empirical likelihood function of the collected dataset yields the logistic loss function (cross-entropy loss function), which is the first term in \eqref{pqr}.

\subsubsection{Sparsity and Variable Selection}\label{sec211}
Let $\boldsymbol{\beta}^*$ be the true coefficient vector of PLR. The set of non-zero coefficients is defined as $\mathcal{A} = \{j:\beta_j^* \ne 0 \}$, and further assume that $|\mathcal{A}| = s < d$. Under this assumption, the coefficients are referred to as sparse, meaning that some elements in  $\boldsymbol{\beta}^*$ are zero and the others are non-zero. Similarly, the design matrix and the estimator can be written as $\bm X = (\bm X_{\mathcal{A}}, \bm X_{\mathcal{A}^c}) = (\bm x_1^\top,\bm x_2^\top, \dots, \bm x_n^\top)^\top$ and $\hat{\boldsymbol{\beta}} =  (\hat{\boldsymbol{\beta}}_{\mathcal{A}}^\top, \hat{\boldsymbol{\beta}}_{\mathcal{A}^c}^\top)^\top$, respectively.
Then, let $I_\mathcal{A} = I_\mathcal{A}(\boldsymbol{\beta}^*_{\mathcal{A}}, \mathbf{0})$ denote the Fisher information under the condition $\bm{\beta}^*_{\mathcal{A}^c} = 0$, and it is obvious that $I_\mathcal{A}$ is the submatrix of $I(\boldsymbol{\beta}^*)$ corresponding to the rows and columns of the set $\mathcal{A}$.   

The method of selecting non-zero coefficients is called variable selection. In logistic regression, a commonly used variable selection method is to add a regularization term as shown in \eqref{pqr}.  The estimator  of  ALASSO-PLR is given as follows:
\begin{align}\label{apqr} 
\hat{\bm \beta } \in   \arg \min_{\bm \beta} \left \{ -\sum_{i=1}^{n} \left[ y_i \bm \beta^\top \bm x_i  -  \log(1  + e^{\bm \beta^\top \bm x_i }) \right] + \sum_{j=1}^{d} {\lambda_j}|\beta_j| \right \}.
\end{align}     
The first term is derived from substituting \eqref{p1} into the first term of \eqref{pqr}. The second term is the ALASSO regularization term first  proposed by \cite{Zou2006TheAL}. Specifically, the single regularization parameter $\lambda$ in LASSO is replaced by parameter vector $ \bm \lambda = (\lambda_1,\dots,\lambda_d)^\top$. Here, $\lambda_j$ varies with the absolute value of $\beta_j$. Generally speaking, the larger $|\beta_j|$ is, the smaller $\lambda_j$ will be.  
Zou \cite{Zou2006TheAL} put forward a weighted strategy, which is expressed as
\begin{align}\label{w1}
 \bm\lambda = \frac{\lambda }{|\bm \beta^{'}|^{\gamma}}.    
\end{align}
Here, $\lambda$ represents the regularization parameter, $\bm \beta^{'}$ is a $\sqrt{n}$-consistent estimator of $\bm \beta^*$, and $\gamma > 0$ is a positive constant. The $\sqrt{n}$-consistent estimator of $\bm \beta^*$ is defined as
\begin{align}\label{cons1}
   \|{\boldsymbol{\beta}}^{'} - \boldsymbol{\beta}^*\|_2 = O_P(n^{-1/2}). 
\end{align}

A consistent estimator generally produces  precise coefficient estimates, especially in the big-data scenario (i.e., when $n$ is extremely large). Subsequently, the weighted strategy presented in \eqref{w1} can assign small regularization parameters to coefficients with large absolute values and large regularization parameters to those with small absolute values.  A natural choice for $\bm \beta^{'}$ is the ordinary logistic regression estimator. Unless otherwise specified below, $\bm \beta^{'}$ will denote the ordinary logistic regression estimator. 

Next, we will review the statistical theory for ALASSO-PLR without noise. For the PLR with other regularization terms, please refer to  the comprehensive sparse statistics learning book \cite{Hastie2015StatisticalLW}, and therein references.

\subsubsection{Oracle Property}\label{sec212}
As emphasized in Theorem 4  of \cite{Zou2006TheAL}, the PLR estimator is called have  oracle property, if the estimator (asymptotically) satisfies  the following two properties 
\begin{enumerate}
    \item \textbf{Consistency in variable selection}: It accurately identifies the true subset model, i.e., \(\{j : \hat{\boldsymbol{\beta}}_{j} \neq 0\} = \mathcal{A}\).

    \item \textbf{Asymptotic normality}: It achieves the optimal estimation rate, expressed as \(\sqrt{n}(\hat{\boldsymbol{\beta}}_{\mathcal{A}} - \boldsymbol{\beta}_{\mathcal{A}}^{*}) \xrightarrow{d} \text{N}(0, I_\mathcal{A}^{-1})\), where \(\xrightarrow{d}\) denotes convergence in distribution.
\end{enumerate}
The importance of the oracle property has been emphasized in the literature \cite{Fan2001VariableSV}  and \cite{Fan2004NonconcavePL}, with the argument that an effective procedure should ideally possess these characteristics.

These two properties are not possessed by ordinary logistic regression and LASSO-PLR in general situations. Specifically, the ordinary logistic regression estimator  $\bm \beta^{'}$  has $\sqrt{n}$-consistency in \eqref{cons1} and the following asymptotic distribution
\begin{align}
    \sqrt{n}({\boldsymbol{\beta}}^{'}- \boldsymbol{\beta}^{*}) \xrightarrow{d} \text{N}(0, I(\boldsymbol{\beta}^{*})^{-1}).
\end{align}
However, its coefficients do not have sparsity, that is, the estimated coefficients are not exactly zero. In other words, although the estimated values of the coefficients that are truly zero are small, they are not absolutely zero.  Therefore, $\bm \beta^{'}$ without sparsity can not have the above two properties. 
On the other hand, the LASSO-PLR estimator can achieve sparsity. However, since the LASSO regularization applies the same regularization parameter to all estimation coefficients, it will inevitably result in an inability to accurately select non-zero coefficients or cause a significant bias. 
In detail, when $\lambda$ is not large enough, the LASSO-PLR estimator may fail to shrink certain irrelevant variables to zero, which means it does not satisfy the first property. As the first property is a prerequisite for the second one, the LASSO-PLR estimator (with a not large enough $\lambda$) does not satisfy the second property either.
When $\lambda$ is large, the LASSO-PLR estimator will shrink all irrelevant coefficients to zero. Nevertheless, it will also shrink some coefficients with relatively small absolute values to zero while simultaneously over-compressing other coefficients, leading to significant and non-negligible estimation errors. As a result, it does not possess the two properties.  The idea of applying adaptive regularization terms  in \eqref{w1} to coefficients in ALASSO-PLR is to address this drawback of LASSO-PLR.

\subsection{Manual Labeling}\label{sec22}
The issue of noise in manual labels in supervised learning is an important research area.  Early methods treated the truth labels as latent variables and estimated them based on the observed noisy labels \cite{Smyth1994InferringGT}. In recent years, research has shifted its focus to constructing generative models for noisy labels to train classifiers, see \cite{Bouveyron2009RobustSC} and \cite{Yan2010ModelingAE}. In crowdsourced label analysis, methods like those in \cite{Raykar2012EliminatingSA} and \cite{Zhang2013LearningBA} have been developed to filter malicious annotators and rank them by domain expertise.  It is commonly assumed that annotators have comparable skills and no malicious behavior, which is reasonable in medical research but may not hold in low-threshold crowdsourcing applications.

Research often assumes that label noise is uniformly distributed within classes (see \cite{Jin2002LearningWM} and \cite{Raykar2009SupervisedLF}), but some studies have considered feature-dependent noise \cite{Yan2010ModelingAE}. Under fixed class-conditional noise, logistic regression may be inconsistent \cite{Bi2010TheEO}, while Song et al. \cite{Song2019ConvexAN} showed that under class-conditional uniform noise, noisy labels can be regarded as responses in a modified generalized linear model. Moreover, the phenomenon that label noise is more concentrated around classification decision boundaries has also been explored in robust estimation \cite{Blanchard2016ClassificationWA} , but there is limited research on relevant statistical modeling and quantification of information loss. In practice, information loss, human errors, and data entry mistakes may lead to label noise. To simplify theoretical analysis, the authors in \cite{Ahfock2021HarmlessLN} commonly assumed that label errors stem solely from classification difficulty, and a finite-mixture model was then employed to relax this assumption.   In this paper, we also follow this approach. 

In \cite{Ahfock2021HarmlessLN}, the authors indicated that when the source of label noise is only related to classification difficulty, this noise does not cause loss of statistical information and is therefore  beneficial. Here, we apply their theoretical results to the PLR scenario.  When discussing the variable selection properties in Section \ref{sec21}, it is necessary to assume that $\{ (\bm{x}_i, {y}_i) \}_{i = 1}^{n}$ are independent and identically distributed. Therefore, we only need to discuss a single sample $(\bm{x}_i, {y}_i)$.

\subsubsection{Useful Label Noise}
For a given $\bm{\beta}^{*}$, the posterior distributions for the two categories of logistic regression are as follows: 
\begin{equation}
\begin{aligned}\label{pop}
& \Pr(y_i = 1 \mid \bm x_i) = \frac{1}{1 + \exp({- \bm{x}_i ^\top \bm{\beta}^* })} \\ 
& \Pr(y_i = 0 \mid \bm x_i) = \frac{\exp({- \bm{x}_i ^\top \bm{\beta}^* })}{1 + \exp({- \bm{x}_i ^\top \bm{\beta}^* })}.
\end{aligned}  
\end{equation}
Let $ {\bm \Pr} = (\Pr(y_i = 1 \mid \bm x_i), \Pr(y_i = 0 \mid \bm x_i)) = (\Pr_1,\Pr_2)$.

To capture the subjective nature of manual labeling, assume that the annotated label $\tilde{y}_i$ is a Binomial random variable distributed according to the posterior probabilities of class membership \eqref{pop}, i.e., 
\begin{align}\label{p11}
\tilde{y}_i \mid  \bm{x}_i \sim \text{Binomial}( 1, {\Pr}_1 ).
\end{align}
Within this labeling model, the probability of a \textbf{labeling error} depends on the posterior class probabilities in \eqref{pop}. To be specific, 

\begin{equation}
\begin{aligned}
\label{noiseerror}
& \Pr( \tilde{y}_i \neq {y}_i \mid \bm{x}_i) = 1 - \Pr(\tilde{y}_i = {y}_i \mid \bm{x}_i) \\
= & 1 - \Pr(\tilde{y}_i = 0, {y}_i =0 \mid \bm{x}_i) - \Pr(\tilde{y}_i = 1, {y}_i =1 \mid \bm{x}_i) \\ 
= \ & 1 - \Pr(y_i = 0 \mid \bm{x}_i)^2 - \Pr(y_i = 1 \mid \bm{x}_i)^2 \\
= \ & 1 - {\Pr}_1^2 - {\Pr}_2^2 . 
\end{aligned}    
\end{equation}

This means that when the sample size $n$ is particularly large, the number of false labels in this manual labeling method will also be significant.
Since  $\Pr_1 + \Pr_2=1$, $\Pr_1 \ge 0$ and  $\Pr_2 \ge 0$, then the value of $\Pr( \tilde{y}_i \neq {y}_i \mid \bm{x}_i)$ reaches its maximum when $\Pr_1 = \Pr_2 = 0.5$.   This conclusion is consistent with the conclusion obtained using the information entropy   of logistic regression or PLR.   The information entropy is defined as
\begin{align}
   H({\Pr}_{1} ) =  - {\Pr}_{1} \cdot \log_2 {\Pr}_{1} - (1 - {\Pr}_{1})\cdot \log_2(1 - {\Pr}_{1})
\end{align}
The entropy function \(H({\Pr}_{1})\) is a concave function of \({\Pr}_{1}\). It reaches its maximum value of 1  when \({\Pr}_{1} = 0.5\). This indicates that the entropy is at its highest when the model is most uncertain, that is, when the probabilities of the two classes are equal. 

The manual labeling model \eqref{p11} is of particular interest because the joint distribution of the manual label $\tilde{y}_i$ and the feature vector $\bm{x}_i$ is identical to that of the truth label $y_i$ and the feature vector $\bm{x}_i$. The joint distributions $(\bm{x}_i, \tilde{y}_i) \sim \tilde{f}(\bm{x}_i, \tilde{y}_i; \bm \beta^{*})$ and $(\bm{x}_i, {y}_i) \sim {f}(\bm{x}_i, {y}_i; \bm \beta^{*})$ are equal, as shown below:
\begin{equation}
\begin{aligned}
 \tilde{f}(\bm{x}_i, \tilde{y}_i; \bm \beta^{*}) = & f(\bm{x}_i; \bm \beta^{*}) \tilde{f}(\tilde{y}_i |\bm{x}_i; \bm \beta^{*}) =  f(\bm{x}_i; \bm \beta^{*}) {f}(\tilde{y}_i |\bm{x}_i; \bm \beta^{*}) = {f}(\bm{x}_i, \tilde{y}_i; \bm \beta^{*}), 
\end{aligned}   
\end{equation}
where the second equation holds due to the fact $\tilde{f}(\tilde{y}_i |\bm{x}_i; \bm \beta^{*}) = {f}(\tilde{y}_i |\bm{x}_i; \bm \beta^{*})$. This fact is true since the manual labels $\tilde{y}_i$ are sampled in accordance with the posterior probabilities of class membership in \eqref{p11}.  Consequently, the joint distribution of $(\bm{x}_i, \tilde{y}_i)$ is identical to that of $(\bm{x}_i, {y}_i)$.

Based on the above discussion,  given $n$ independently and identically distributed observations from the manual labeling model, $\{ (\bm{x}_i, \tilde{y}_i) \}_{i = 1}^{n}$, the distribution of the maximum likelihood estimate will be the same as when using a dataset with the truth labels $\{ (\bm{x}_i, {y}_i) \}_{i = 1}^{n}$. 
It then follows that the maximum likelihood loss of a classifier trained using $\{\tilde{y}_i\}_{i=1}^n$ will be the same as that of a classifier trained using the truth labels $\{{y}_i\}_{i=1}^n$. This serves as the theoretical guarantee that label noise can be useful when the label noise process is determined by the posterior probabilities of class membership in \eqref{pop}. As a result, when multiple experts provide manual labels, the resulting manual labels can provide more statistical information compared to the truth labels. In the following, we will explain the method for simulating the generation of manual labels from multiple experts.  It is worth noting that this generation method is particularly simple and has a low computational burden.

\subsubsection{Multi-Expert Manual Labeling}\label{sec31}
Suppose there are $m$ independent individuals in the labeling group, and each provides a label for each of the $n$ feature vectors in the dataset. In an ideal scenario,  given the $n$ independent and identically distributed   observation sample $\{\bm x_i,y_i \}_{i=1}^n$,
the  manually assigned labels are considered as $mn$ independent observations from the model:
\begin{align}\label{DB}
\tilde{\bm Y}_{ik} \mid \bm{x}_i \sim \text{Multinomial}(1, \bm{\Pr} ), i = 1,2,\dots,n, \  k = 1, \ldots, m. 
\end{align}
Then, $\tilde{\bm Y}_{i} = (\tilde{Y}_{i1}, \tilde{Y}_{i2}, \dots,  \tilde{Y}_{im})$ represents a $2 \times m$ matrix, which corresponds to the manual labeling of the $i$-th label carried out by $m$ independent experts. By summing the rows of matrix $\tilde{\bm Y}_{i}$, we can obtain  $\mathbb{S}_i = ( S_i^{(1)},  S_i^{(0)})$, where $ S_i^{(0)}$ is the number of experts who cast  $\tilde{y}_i = 0$,  and $ S_i^{(1)}$ is the number of experts who cast  $\tilde{y}_i = 1$.
As a result,  $\mathbb{S}_i$  has the distribution  
 \begin{align}\label{d1}
\mathbb{S}_i \sim \text{Multinomial}(m, \bm{\Pr} ). 
\end{align}
When $m=1$, the $S_i^{(1)}$ generated by \eqref{d1} is the same as the $\tilde{y}_i$ generated in \eqref{p11}.
                                                                
Nevertheless, the aforementioned scenario merely depicts an idealized state, and its realization in practical applications is highly improbable. \textbf{To  loosen the assumption} regarding the accuracy of manual labeling, we employ a random  effects model for the counts $\mathbb{S}_i$. The Dirichlet-Multinomial model \cite{Johnson1997DiscreteMD} is an extension of the Multinomial distribution. It can be utilized to take into account the fact that the group of experts ($m>1$) does not possess precise knowledge of the posterior probabilities of $\bm \Pr$. We can introduce the overdispersion parameter $\alpha_0 \in \mathbb{R}^+$ to formulate the proposed Dirichlet - Multinomial model for the aggregated noisy manual labels as follows:
\begin{align}\label{d2}
\mathbb{S}_i \sim \text{Dirichlet - Multinomial}(m, \alpha_0\bm \Pr),
\end{align}
where $\mathbb{E}(S_i^{(1)}) = m {\Pr}_1$, $\mathbb{E}(S_i^{(0)}) = m {\Pr}_2$, $\text{Var}(S_i^{(1)}) = m {\Pr}_1(1 - {\Pr}_1)\frac{m+\alpha_0}{1+\alpha_0}$, and $\text{Var}(S_i^{(0)}) = m {\Pr}_2(1 - {\Pr}_2)\frac{m+\alpha_0}{1+\alpha_0}$. 
When $m = 1$, $S_i^{(1)}$ can only take values of either 0 or 1. In this case, $\mathbb{E}(S_i^{(1)}) = {\Pr}_1$ and $\text{Var}(S_i^{(1)}) = {\Pr}_1(1 - {\Pr}_1)$. That is, for any value of $\alpha_0$, $S_i^{(1)} \sim \text{Binomial}(1, {\Pr}_1)$, which is consistent with the $\tilde{y}_i$ generated in \eqref{p11}. 

The Dirichlet-Multinomial model \eqref{d2} can also  be hierarchically represented as follows:
\[
\begin{cases}
\bm p_i = ( p_{i}^{(1)}, p_{i}^{(0)}) \sim \text{Dirichlet}(\alpha_0 \bm \Pr), \\
\mathbb{S}_i  \mid \bm p_i \sim \text{Multinomial}(m, \bm p_i),
\end{cases}
\]
where $\mathbb{E}(\bm{p}_i)= \bm \Pr = ({\Pr}_1,{\Pr}_2)$ and \(\text{Var}(p_{i}^{(1)})=\text{Var}(p_{i}^{(0)})=\frac{{\Pr}_1{\Pr}_2}{\alpha_0 + 1}\).  
Note that when the Dirichlet distribution has two-dimensional parameters, it degenerates into a Beta distribution, which is the conjugate prior of the binomial distribution. 
Consequently, when ${\Pr}_1 = {\Pr}_2$, the Dirichlet distribution is symmetric. Moreover, when $\alpha_0 {\Pr}_1 = \alpha_0 {\Pr}_2 = 1$, the Dirichlet distribution degenerates into a uniform distribution on the interval $[0, 1]$. 

The overdispersion parameter $\alpha_0$ regulates the degree to which the samples of $\bm p_i $ cluster around the true posterior class-probabilities $\bm \Pr$. Two extreme situations will be explained below.
\begin{itemize}
    \item  As $\alpha$ tends to 0 ($\bm p_i$ is far from $\bm \Pr$), the probability of the Dirichlet (Beta)  distribution is concentrated at the two cases $\bm p_i = \{0, 1\}$ or $\bm p_i = \{1, 0\}$. 
This phenomenon is also manifested in the variance. Specifically, $$\lim\limits_{\alpha_0 \to 0} \text{Var}(p_{i1}) = \lim\limits_{\alpha_0 \to 0} \text{Var}(p_{i2}) = {\Pr}_1 {\Pr}_2.$$ That is to say,  $\bm p_i $  will converge to  a 0-1  distribution. Consequently,
 when $\bm p_i$ takes the values $\{0, 1\}$ or $\{1, 0\}$, it implies that multiple independent experts will uniformly assign all the manual labels $\tilde{y}_i$ to be either 1 or 0. In other words, the value of $S_i^{(1)} (S_i^{(0)})$ will be either $m(0)$ or $0(m)$ .
    \item  When $\alpha_0$ approaches infinity ($\bm p_i$ is nearly equal to $\bm \Pr$), the excess variance vanishes, and the Dirichlet-Multinomial(Beta-Multinomial) model \eqref{d2} converges to the Multinomial distribution \eqref{d1}. 
\end{itemize}
The Dirichlet-Multinomial model not only maintains a relationship between the probability of a labeling error and classification difficulty but also relaxes the assumptions about the accuracy of manual annotations compared to \eqref{d1}. 
This scenario is analogous to that in Section 2.3 of \cite{Ahfock2021HarmlessLN}, where we explicitly establish a link between the probability of labeling error and the classification difficulty through the posterior class-probabilities. 
\section{Variable Selection and Relative Efficiency with  Manual Labeling}\label{sec3}
In this section, we  first develop the corresponding statistical theories, which are similar to those in Section \ref{sec21}, under the manual labeling stated in Section \ref{sec22}. Then, leveraging the derived asymptotic distribution, we further compute the asymptotic relative efficiency of ALASSO-PLR under manually-labeled (with noise) data in comparison to that under truth labels. 

Let $\bm{X} = (\bm{x}_1^\top, \bm{x}_2^\top, \ldots, \bm{x}_n^\top)^\top$ denote the $n$ observed feature vectors, $\bm{y} = (y_1, y_2, \ldots, y_n)^\top$ denote the $n$ true labels, $\tilde{\bm{Y}} = (\tilde{\bm{Y}}_1^\top, \tilde{\bm{Y}}_2^\top, \ldots, \tilde{\bm{Y}}_n^\top)^\top$ be the manual labelings of the $n$ items by $m$ independent experts, and $\bm{S} = ( S_1^{(1)},  S_2^{(1)}, \ldots,  S_n^{(1)})^\top$ represent the vote counts for $ (\tilde{y}_1 =1, \tilde{y}_2 =1, \ldots, \tilde{y}_n =1 )^\top$ obtained from the manual labels.   The density function $ f(  S_i^{(1)} \mid \bm x, \bm \beta )$  can be expressed as 
\begin{align}\label{den2}
  f( S_i^{(1)} & \mid \bm x, \bm \beta ) = \left({p}(\bm x, \bm \beta)\right)^{ S_i^{(1)}} \left(1-{p}(\bm x, \bm \beta)\right)^{m - { S_i^{(1)}}}.  
\end{align}

Based on  \eqref{apqr} and \eqref{w1}, we can derive the loss function of ALASSO-PLR, denoted as $L(\bm{X}, \bm{y}, \bm{\beta})$, which is defined as 
\begin{align}\label{plr1} 
    -\sum_{i = 1}^{n} \left[ y_i \bm{\beta}^\top \bm{x}_i - \log\left(1 + \exp(\bm{\beta}^\top \bm{x}_i)\right)\right] 
    + \lambda\sum_{j = 1}^{d} w_j|\beta_j|, \notag
\end{align}
where $w_j = 1/|\beta^{'}_j|^{\gamma}$.
The  negative  empirical log-likelihood function using the aggregated  noisy manual labels  $\bm{S}$ (its density function is in \eqref{den2}) with ALASSO regularization is
\begin{equation}
\begin{aligned}\label{plr2} 
   L(\bm{X}, \bm{S}, \bm{\beta}) = -\sum_{i=1}^{n} \left[ S_i \bm{\beta}^\top \bm{x}_i  -  m \log\left(1  + \exp(\bm{\beta}^\top \bm{x}_i) \right) \right] + \lambda \sum_{j=1} ^{d} w_j|\beta_j|,
\end{aligned}     
\end{equation}   
where the first term in the above equation was used in \cite{Ahfock2021HarmlessLN} to derive the asymptotic distribution with aggregated  noisy manual labels in ordinary logistic regression (without a regularization term). Notably, when \(S_i\) only utilizes \(y_i\),  $L(\bm{X}, \bm{S}, \bm{\beta})$ will degenerate into $L(\bm{X}, \bm{y}, \bm{\beta})$. To distinguish $\hat{\bm \beta}$ (the estimator of ALASSO-PLR), let $\tilde{\bm \beta} \in  \arg \min \limits_{\bm \beta} L(\bm{X}, \bm{S}, \bm{\beta})$.

\subsection{Variable Selection and Asymptotic Normality}
Next, analogous to Section \ref{sec21}, we will derive the statistical theoretical properties of \eqref{plr2} {in the presence of label noise}. Before presenting statistical theory, we need the following two assumptions,
\begin{itemize}
  \item (A1) The Fisher information matrix $I(\boldsymbol{\beta}^*)$ is finite and positive definite.

  \item (A2) Let $\phi({\bm{\beta}^\top \bm{x}}) =  m \log\left(1  + \exp({\bm{\beta}^\top \bm{x}})\right)$, where 
 $\bm x$ is a $d$-dimensional predictor variable and $\bm{x} = (x_1,x_2,\dots,x_d)^\top$. There is a sufficiently large enough open set \( \mathcal{O} \) that contains the true coefficient \( \bm \beta^* \) such that \( \forall \, \bm \beta \in \mathcal{O} \),
\begin{equation}
\begin{aligned}
   |\phi'''({\bm{\beta}^\top \bm{x}})| \leq M(\bm{x}) < \infty  \quad \text{and} \quad \mathbb{E}[M(\bm{x})|x_j x_k x_l|] < \infty, \text{for all} 1 \leq j, k, l \leq d. 
\end{aligned}
\end{equation}
\end{itemize}
These conventional conditions have been widely used to establish the oracle property of PLR with convex \cite{Zou2006TheAL} and non-convex (see \cite{Fan2001VariableSV} and \cite{Fan2004NonconcavePL}) regularization terms.  They are easy to meet. For details, please refer to the discussion in Section 3.2 of \cite{Fan2001VariableSV}. 

Next,  we will give the oracle property of ALASSO-PLR estimator under manual labeling (with noise). The proof is placed in supplementary material 1.1.
\begin{thm}[Oracle Property]\label{th2}
Let $\tilde{\mathcal{A}} = \{j : \tilde{\boldsymbol{\beta}}_{j} \neq 0\}$. Suppose that $\frac{\lambda}{\sqrt{n}} \to 0$ and $\lambda n^{(\gamma - 1)/2} \to \infty$. Then, under the same assumptions (A1)-(A2), the ALASSO-PLR estimator $\tilde{\boldsymbol{\beta}}$ under manual labeling (with noise) must satisfy: 
\begin{enumerate}
    \item[(1)] \textbf{Consistency in variable selection}: \( \lim \limits_{n \to \infty} \Pr(\tilde{\mathcal{A}} = \mathcal{A} )  = 1 \).

    \item[(2)] \textbf{Asymptotic normality}: $$\sqrt{n} \left( \tilde{\boldsymbol{\beta}}_{\mathcal{A}} - \boldsymbol{\beta}^{*}_{\mathcal{A}} 
    \right) \xrightarrow{d} N\{0, \frac{m+ \alpha_0}{m(1+  \alpha_0)}I_\mathcal{A}^{-1}\}.$$ 
    
\end{enumerate}
\end{thm}

Under the Dirichlet-Binomial model \eqref{d2}, the ALASSO-PLR method using the noisy labels $\tilde{\bm Y}$ gives a consistent estimator of the true coefficient $\bm \beta^*$ with the oracle property.  When $m =1$ or $\alpha_0 \to 0$,  $\sqrt{n} \left( \tilde{\boldsymbol{\beta}}_{\mathcal{A}} - \boldsymbol{\beta}^{*}_{\mathcal{A}} 
    \right) \to N\{0, I_\mathcal{A}^{-1}\}$, which is the same as the asymptotic distribution of PLR (under truth labels) in Section \ref{sec212}.
    The following will briefly explain the reasons behind this situation. 
 \begin{itemize}
    \item  When $m = 1$,  there is only one expert providing manual labels. In this case, the statistical information contained is the same as that of the truth label $\bm y$. Therefore, the asymptotic distributions are identical.
    \item  When $\alpha_0$ approaches 0, as discussed in Section \ref{sec31}, the manual labels provided by all experts are identical. This situation is equivalent to having a single expert manually assign the labels.
\end{itemize}
In fact, when \(m > 1\) and \(\alpha_0 > 0\), more statistical information is provided compared to the case of \(m = 1\) and \(\alpha\to0\).  When \(m = 1\), that is, the manual label from one expert provides the same statistical information as the truth label. As the number of experts increases, more statistical information will be provided. When \(\alpha_0\) is larger, \(\bm p_i\) will be closer to the true posterior probability, and the manual labels will be more accurate compared to the situation where \(\alpha_0\to0\).  This means that the ALASSO-PLR estimator can benefit from manual labels (with noise). Subsequently,  we will quantify this phenomenon using asymptotic relative efficiency.

\subsection{Asymptotic Relative Efficiency}\label{sec32}
The asymptotic relative efficiency of ALASSO-PLR with manual labels, when compared to its efficiency under the truth label, is defined as:
\begin{align}\label{are}
    \text{ARE}=\lim_{n\rightarrow\infty}\frac{\mathbb{E}\{\text{err}(\hat{\bm \beta}; \bm \beta^*)\}-\text{err}(\bm \beta^*; \bm \beta^*)}{\mathbb{E}\{\text{err}(\tilde{\bm \beta}; \bm \beta^*)\}-\text{err}(\bm \beta^*; \bm \beta^*)}.
\end{align}
Here, \(\text{err}(\bm \beta; \bm \beta^*) =  \sum \limits_{i =1}^n \Pr(y_i^{'} \ne y_i \mid  \bm \beta^* ) \) represents the conditional error rate when the estimate \(\bm \beta\) is employed while the  true coefficient is \(\bm \beta^*\), where $y_i^{'}$ is the predicted label based on the estimated coefficient \(\bm \beta\).  ARE is an indicator in statistics used to compare the efficiency of two estimators. It depicts the relative performance of the two estimators as the sample size approaches infinity. Specifically, the ARE measures the ratio of the sample sizes required for the two estimators to achieve the same level of precision. 

Based on  Theorem \ref{th2}, we can obtain the following asymptotic relative efficiency. Its proof is in supplementary material 1.2.
\begin{thm}[Asymptotic Relative Efficiency]\label{th3}
The asymptotic relative efficiency of ALASSO-PLR with manual labels (with noise), as compared to that under the truth label, is
\begin{align}\label{are1}
\text{ARE} = \frac{m(1+\alpha_0)}{m+\alpha_0}.
\end{align}
\end{thm}

In the context of statistical analysis, the behavior of the ARE shows a clear relationship with the parameters \(m\) and \(\alpha_0\). This relationship can be broken down into the following key points:

\begin{enumerate}
    \item \textbf{Baseline or equilibrium state}: When \(m = 1\), the \(\text{ARE}\) is exactly 1. Similarly, as \(\alpha_0\) approaches 0, the \(\text{ARE}\) converges to 1. This indicates a baseline or equilibrium state of relative efficiency between the estimators being compared.
    \item \textbf{Superior efficiency range}: When either \(m>1\) or \(\alpha_0 > 0\), the \(\text{ARE}\) exceeds 1. This implies that within these parameter ranges, the ALASSO-PLR model with manual labels exhibits superior efficiency compared to the one under the truth label.
    \item \textbf{Influence of \(m\) at extreme values}: As \(m\) tends towards infinity, the \(\text{ARE}\) becomes \((1 + \alpha_0)\). This suggests that in such extreme cases, \(\alpha_0\) plays a decisive role in determining the relative efficiency. At this point, the larger the value of $\alpha_0$ (indicating a clearer understanding of the posterior probability), the higher the efficiency of the ALASSO-PLR with manual labels.
    \item \textbf{Influence of \(\alpha_0\) at extreme values}: When \(\alpha_0\) approaches infinity, the \(\text{ARE}\) is equal to \(m\). This highlights that \(m\) dominates in dictating the relative efficiency under this extreme scenario. At this point, as the number of experts increases ($m$ increases), the higher the efficiency of the ALASSO-PLR with manual labels.
\end{enumerate}
From the above discussion, it can be concluded that the clearer our understanding of the posterior probability and the larger the number of experts, the higher the asymptotic relative efficiency estimated using the manually-generated labels (with noise) in Section \ref{sec2} of this paper. This indicates that such noise is useful for variable selection in ALASSO-PLR. 

\subsection{Discussions and Applications}\label{sec33}
Theorems \ref{th2} and \ref{th3} reveal that when manual labels are generated according to Section \ref{sec22}, even if noise occurs, it will not affect the variable selection results and improve asymptotic efficiency of ALASSO-PLR. In Section \ref{sec22}, we allow the manual labels  do not possess precise knowledge of the posterior probabilities of $\bm \Pr$.  This means that when generating manual labels, we do not necessarily need the true coefficient $\bm \beta^*$. An appropriate estimate of $\bm \beta^*$ is sufficient, of course, the more accurate the estimate (the larger the $\alpha_0$), the better the performance.

Therefore, the manual labeling method proposed in this paper has two applications for PLR as follows.
\begin{enumerate}
    \item  \textbf{Enhancing  variable selection in large-scale PLR.} For the collected data $\{\bm X, \bm y\}$, we can obtain a ordinary PLR estimator $\hat {\bm \beta}$ to replace $\bm \beta^*$ in \eqref{DB} generating manual labels. The use of these generated labels can enhance the coefficient estimation performance of PLR. The generation of these manual labels is easy, just a multinomial distribution, and does not require too much computational burden.
     \item \textbf{Filling in the missing labels for large-scale PLR.} When there are missing or Incomplete data in the collected data \cite{Xu2025IncompleteLD}, an estimated $\hat {\bm \beta}$ can be obtained from the existing data and then used to substitute this $\hat {\bm \beta}$ into the manual label generation model in \eqref{DB}. The manually generated labels can be employed to fill in the missing labels and improve the coefficient estimation performance of PLR.
\end{enumerate} 

\section{Partition-Insensitive Parallel Algorithm}\label{sec4}
Recall the optimization expression in \eqref{plr2}. Let $\bm{r} = (r_1, r_2, \dots, r_n)^\top = \bm{X}\bm{\beta} = (\bm{\beta}^\top\bm{x}_1, \bm{\beta}^\top\bm{x}_2, \dots, \bm{\beta}^\top\bm{x}_n)^\top$. Then, the equivalent constrained optimization form of (\ref{plr2}) is written as
\begin{align}\label{cplr} 
\min_{\bm{\beta},\bm{r}} \ & \sum_{i = 1}^{n} \left[ -S_i r_i + m \log\left(1 + \exp(r_i)\right) \right] + \lambda \|\bm{w} \odot \bm{\beta}\|_1, \notag \\
\text{s.t.} \ \ & r_i = \bm{\beta}^\top\bm{x}_i, \quad i = 1, 2, \dots, n.
\end{align}
where $\odot$ denotes Hadamard product. It is easy to see that (\ref{cplr}) is an optimization formula with equality constraints and no coupling terms for the optimization variables, which can be solved by using the alternating direction method of multipliers (ADMM) in \cite{Boyd2010DistributedOA}.

\subsection{ADMM for ALASSO-PLR with Manual Labels}
ADMM is an algorithm widely used in convex optimization problems (see \cite{Boyd2010DistributedOA,Jiao2026MarginalDR}) and can be applied to solve large-scale convex optimization problems with separable structures. This algorithm decomposes complex optimization problems into multiple simple subproblems and solves these subproblems through alternating iterations to gradually approach the optimal solution of the original problem.  The augmented Lagrangian form of (\ref{cplr}) is
\begin{align}\label{al}
L_\mu(\bm \beta, \bm r, \bm u) = \sum_{i=1}^{n} \left[ -S_i r_i +  m \log\left(1  + \exp({r_i})\right)  \right] + \lambda \|\bm w \odot \bm \beta \|_1  - \bm u^\top (\bm X \bm \beta  - \bm r ) +  \frac{\mu}{2}  \|  \bm X \bm \beta - \bm r  \|_2^2,
\end{align}    
where $\bm u$  is dual variable corresponding to the linear constraint, and $\mu>0$ is a given  augmented parameter.
The iterative scheme of ADMM for (\ref{cplr}) is
\begin{equation}\label{twoupadmm}
\left\{ \begin{array}{l}
\bm \beta^{k+1}  \leftarrow  \mathop {\arg \min }\limits_{\bm \beta} \left\{ L_\mu (\bm \beta, \bm r^k,  \bm u^k) \right \};\\
\bm r^{k+1}  \leftarrow  \mathop {\arg \min }\limits_{\bm r} \left\{ L_\mu (\bm \beta^{k+1}, \bm r, \bm u^k)\right \}; \\
\bm u^{k+1}  \leftarrow  \bm{u}^{k} - \mu(\bm X \bm \beta^{k+1} - \bm r^{k+1} ).
\end{array} \right.
\end{equation}
For the subproblem of updating $\bm \beta^{k + 1}$ in (\ref{twoupadmm}), we can rearrange the optimization equation and omit constant terms irrelevant to the target variable $\bm \beta$, yielding the following iterative formula,
\begin{equation}
\begin{aligned}\label{beta}
\bm \beta^{k+1} \leftarrow \arg \min_{\boldsymbol\beta}  & \left\{ \lambda \| 
 \bm w \odot \bm \beta \|_1 + \frac{\mu}{2}  \|  \bm X \bm \beta - \bm r^k  - \bm u^k/\mu \|_2^2  \right\}.
\end{aligned}
\end{equation}

Clearly, the non-identity matrix before the quadratic term of $\bm \beta$ and $\| 
\bm w \odot \bm \beta \|_1$ prevent iterative formula in  (\ref{beta}) from having a closed-form solution. While numerical methods like coordinate descent can solve it, they greatly increase the computational burden. We use the linearization method from \cite{Yuan2020DiscerningTL} and  \cite{Liang2024LinearizedAD} to approximate this optimization problem and obtain a closed-form solution for $\bm \beta^{k+1}$.  Specifically, we suggest adding a proximal term to the objective function in (\ref{beta}) and updating $\bm \beta^{k+1}$ with the following iterative formula,
\begin{align}\label{lbeta}
\bm \beta^{k+1} \leftarrow \arg \min_{\boldsymbol\beta} \left\{ \lambda \|\bm w \odot \bm \beta \|_1 + \frac{\mu}{2}  \|  \bm X \bm \beta - \bm r^k  - \bm u^k/\mu \|_2^2 + \frac{1}{2} \| \bm \beta - \bm \beta^k \|^2_{{\bm E}_{1}} \right\}.
\end{align}
where $\bm E_1 = \eta \bm I_p - \mu \bm X^\top \bm X$ is a positive semidefinite matrix. To guarantee that $\bm E_1$ is a positive definite matrix, $\eta$ should be no less than the largest eigenvalue of $\mu \bm X^\top \bm X$ (denoted as $\|\mu \bm X^\top \bm X\|$).  After rearranging the terms in (\ref{lbeta}), the update rule for $\bm \beta^{k + 1}$ can be expressed as
\begin{align}\label{lbeta2}
\bm \beta^{k+1} \leftarrow \arg \min_{\boldsymbol\beta} \left\{ \lambda \| 
 \bm w \odot \bm \beta \|_1 + \frac{\eta}{2} \| \bm \beta - \bm \beta^k + \frac{\mu \bm X^\top(\bm X \bm \beta^k - \bm r^k  - \bm u^k/\mu) }{\eta} \|_{2}^2 \right\}.
\end{align}    
As a result, the closed-form solution of \eqref{lbeta2} is a weighted soft-thresholding operator, that is, 
\begin{equation}
\begin{aligned}\label{lbeta3}
\bm \beta^{k+1}  \leftarrow \text{sign}\left(\bm \beta^k - \frac{\mu \bm X^\top(\bm X \bm \beta^k - \bm r^k  - \bm u^k/\mu) }{\eta} \right) \odot \left( \left|\bm \beta^k - \frac{\mu \bm X^\top(\bm X \bm \beta^k - \bm r^k  - \bm u^k/\mu) }{\eta} \right| - \lambda \bm w/\eta \right)_+.
\end{aligned}     
\end{equation}   

For the subproblem of updating $\bm r^{k + 1}$ in (\ref{twoupadmm}),  rearranging the optimization formula yields
\begin{align}\label{r} \notag
\bm r^{k+1} \leftarrow \arg \min_{\boldsymbol r} \left\{ \sum_{i=1}^{n} \left[ -S_i r_i +  m \log\left(1  + \exp({r_i})\right)  \right] + \frac{\mu}{2}  \|  \bm X \bm \beta^{k+1} - \bm r  - \bm u^k/\mu \|_2^2  \right\}.
\end{align}    
It is a system of nonlinear equations, and each component in the vector $\bm{r}$ can be solved element-by-element.  Due to the existence of $m\log\left(1  + \exp({r_i})\right)$, there is no closed-form solution for the above iterative formula. The Newton's iteration method is an effective algorithm for solving this system of nonlinear equations. Linearization techniques can also be used, which involves adding a quadratic term, that is, 
\begin{equation}
\begin{aligned}\label{lr0}
\bm r^{k+1} \leftarrow  \arg \min_{\boldsymbol r} \left\{ \sum_{i=1}^{n} \left[ -S_i r_i +  m \log\left(1  + \exp({r_i})\right)  \right] + \frac{\mu}{2}  \|  \bm X \bm \beta^{k+1} - \bm r  - \bm u^k/\mu \|_2^2 + \frac{1}{2} \| \bm r - \bm r^k   \|^2_{{\bm E}_{2}}  \right\},
\end{aligned}    
\end{equation}
where  $\bm E_2 =  (c_0 - o(1))\bm I_p$ with $c_0$ is an arbitrary positive constant. Note, when we use Newton's method to solve the system of linear equations, setting \(c_0 = o(1)\) suffices. Here, the \(o(1)\) is only used for describing the algorithm iteration process, and there is no need to know its specific value. Adding this quadratic term is to remove the remainder of the second-order Taylor expansion of $\phi(r_i) =m\log\left(1  + \exp({r_i})\right)$. To be specific, for each $i$, we have
\begin{align*}
\phi(r_i) + (1 -o(1)) (r_i -r_i^k)^2 & = \phi(r_i^k) + \phi^{'}(r_i^k)(r_i -r_i^k) + \frac{\phi^{''}(r_i^k) + o(1)}{2}(r_i -r_i^k)^2+ \frac{ c_0 -o(1)}{2} (r_i -r_i^k)^2 \\
& = \phi(r_i^k) + \phi^{'}(r_i^k)(r_i -r_i^k) + \frac{\phi^{''}(r_i^k) + c_0}{2}(r_i -r_i^k)^2.
\end{align*}
Then,  by rearranging the optimization equation  (\ref{lr0}), we obtain
\begin{align}\label{lr2} 
\bm r^{k+1} \leftarrow \arg \min_{\boldsymbol r} \left\{  (\phi^{'}(\bm r^k ) - \bm S)^\top\bm r + \frac{1}{2} \| \bm r - \bm r^k   \|^2_{\Phi(\bm r^k)} + \frac{\mu}{2}  \|  \bm X \bm \beta^{k+1} - \bm r  - \bm u^k/\mu \|_2^2  \right\},
\end{align}
where $\phi^{'}(\bm r^k) = (\phi^{'}(r_1^k),\phi^{'}(r_2^k),\dots,\phi^{'}(r_n^k))^\top$ and $\Phi(\bm r^k)$  is a diagonal matrix, and the diagonal elements are $\{\phi^{''}(r_i^k) +c_0 \}_{i=1}^n$. Obviously, \eqref{lr2} is a quadratic function, and the derivative can have the following closed-form solution, 
\begin{equation}
\begin{aligned}\label{lr3}
  \bm r^{k+1} \leftarrow \left[ \Phi(\bm r^k) + \mu \bm I_p   \right]^{-1}  \left[ \mu  ( \bm X \bm \beta^{k+1} - \bm u^k/\mu  ) + \Phi(\bm r^k) \bm r^k  -  \phi^{'}(\bm r^k) + \bm S \right].
\end{aligned}
\end{equation}

To sum up, the iteration of ADMM for (\ref{cplr}) can be described in Algorithm \ref{alg1}. Note that Algorithm \ref{alg1} can also be used to solve ALASSO-PLR under the truth label when $\bm S$ is generated by the real observation label $\bm y$.
\begin{algorithm}\small
\caption{\small{ADMM for ALASSO-PLR with manual labels}}
\label{alg1}
\begin{algorithmic}
\STATE {\textbf{Input:} observation data $\boldsymbol{X}$; manual label $\bm S$; primal variables $\boldsymbol{\beta}^0,\boldsymbol{r}^0$; dual variables $\boldsymbol{u}^0$; augmented parameters $\mu$; penalty parameter $\lambda$; and weight vector $\bm w$.}
\STATE {\textbf{Output:} the total number of iterations $K$,  $\boldsymbol{\beta}^K, \bm r^K, \bm u^K$. }
\STATE {\textbf{while} not converged \textbf{do}}
\STATE {\quad 1. Update $\boldsymbol{\beta}^{k+1}$ using (\ref{lbeta3})},
\STATE {\quad 2. Update $\boldsymbol{r}^{k+1}$ using (\ref{lr3})},
\STATE {\quad 3. Update $\boldsymbol{u}^{k+1}$ using $\bm u^{k+1}  \leftarrow  \bm{u}^{k} - \mu(\bm X \bm \beta^{k+1} - \bm r^{k+1} )$},
\STATE {\textbf{end while}}
\STATE {\textbf{return} solution}.
\end{algorithmic}
\end{algorithm}

The design details of Algorithm \ref{alg1} are entirely novel, even though the linearized ADMM algorithm has been a topic of long-standing discussion (for details, refer to Chapter 3 in \cite{Lin2022AlternatingDM} and the references therein).  Specifically, the existing ADMM algorithms in \cite{Yuan2020DiscerningTL} and \cite{Li2019AcceleratedAD} for PLR constructed equality constraints that replace $\bm \beta$ or its overlapping structure, while Algorithm \ref{alg1} focuses on $\bm X \bm \beta  = \bm r $. The use of constructing this equality constraint is convenient for the design of subsequent partitioning insensitive parallel algorithms. Moreover, Algorithm \ref{alg1} linearizes both $\bm \beta$ and $\bm r$. This two-variable linearization approach has been reported in \cite{Yuan2020DiscerningTL} (Algorithm 1) and \cite{Lin2022AlternatingDM} (Algorithms 3.2, 3.3).  Both Algorithm 1 in  \cite{Yuan2020DiscerningTL}  and Algorithm 3.2 in \cite{Lin2022AlternatingDM}   focus on linearizing the quadratic augmented term. This stands in contrast to Algorithm \ref{alg1} put forward in this paper. Specifically, in Algorithm \ref{alg1} of this paper, the first term solely represents the linearized quadratic augmented term, while the second term is only related to the linearization of the logistic regression loss function.  Algorithm 3.3 in \cite{Lin2022AlternatingDM}   simultaneously linearizes both the loss function and the quadratic augmented term for two subproblems, which also distinguishes it from Algorithm \ref{alg1} in this paper. Moreover, the linearization of its functions requires calculating the Lipschitz constant of the first-order derivative of some functions, differing from the linearization scheme proposed in this paper. 

Next, we will discuss the convergence and convergence rate of Algorithm \ref{alg1}. Its proof can be found in supplementary material 1.3.

\begin{thm}\label{th4}
Let the sequence \( \left\{ \bm v^k= ( \bm{\beta}^k, \bm r^k, \bm u^k) \right\} \) be generated by Algorithm \ref{alg1}.
\begin{enumerate}
\item (Algorithm global convergence). It converges to some  \( \bm v^\infty =  ( \bm{\beta}^\infty, \bm r^\infty, \bm u^\infty ) \)  that belongs to $ \Omega^*$, where  $ \Omega^*$ is the set of optimal solutions for (\ref{cplr}).
\item (Sublinear convergence rate).  For any integer \(K > 0\), we have
\begin{align}\label{p25}
\|\bm v^K - \bm v^{K+1} \|_{\bm{H}}^2 \leq \frac{1}{c_0 \left(K+1\right)}  \|\bm v^0 - \bm v^*\|_{\bm{H}}^2,
\end{align}    
where $\bm v^* \in \Omega^*$, $c_0$ is a positive constant, $\| \bm v \|_{\bm H} = \sqrt{\bm v^\top \bm H \bm v }$ and   $\bm H = \begin{pmatrix}
\bm E_1 & 0 & 0 \\
0 & \mu \bm I_n + \bm E_2 & 0 \\
0 & 0 & \frac{1}{\mu}\bm I_n
\end{pmatrix}$ is a  positive definite matrix.
\end{enumerate}
\end{thm}
\begin{rem}
 Evidently, the $\bm{H}$-norm squared $\|\bm{v}^0 - \bm{v}^*\|_{\bm{H}}^2$ is a positive constant. Consequently, the $\bm{H}$-norm squared $\|\bm{v}^K - \bm{v}^{K + 1}\|_{\bm{H}}^2$ is bounded by $\mathcal{O}\left(\frac{1}{K+1}\right)$, indicating a sublinear convergence rate.  In addition, when we use the proof method in Section 4 of \cite{He2015OnNC}, $c_0$ can be specifically taken as 1.
\end{rem}

Although the differences between Algorithm 3.3 in \cite{Lin2022AlternatingDM} and Algorithm \ref{alg1} have been discussed above, Algorithm \ref{alg1} can be regarded as a special case of Algorithm 3.3 in terms of the linearization strategy. However, regarding the algorithm convergence, the convergence conclusion of Theorem \ref{th4} differs from that presented in Chapter 3.2.1 of \cite{Lin2022AlternatingDM}. To be specific, Theorem 3.7 in  \cite{Lin2022AlternatingDM} discusses the sublinear convergence of the linearized algorithm. However, this sublinear convergence is concerned with the optimization objective function and constraints, while Theorem \ref{th4} in this paper focuses on the iterative point sequence. Moreover, in the course of proving this theorem, we showed that both \( \|\bm{v}^k - \bm{v}^{*}\|_{\bm{H}}^2 \) and \(  \|\bm{v}^k - \bm{v}^{k + 1}\|_{\bm{H}}^2  \) exhibit non-increasing monotonicity. Specifically, we have \( \|\bm{v}^{k + 1} - \bm{v}^{*}\|_{\bm{H}}^2   \leq \|\bm{v}^k - \bm{v}^{*}\|_{\bm{H}}^2 \) and \( \|\bm{v}^{k + 1} - \bm{v}^{k + 2}\|_{\bm{H}}^2   \leq \|\bm{v}^k - \bm{v}^{k + 1}\|_{\bm{H}}^2 \). For more details, refer to the propositions in the supplementary material. 

It is worth mentioning that Theorem 3.9 in \cite{Lin2022AlternatingDM} explores the linear convergence of the linearized algorithm, but it requires both objective functions to be strictly strongly convex. Nevertheless, in the optimization objective function of this paper, the adaptive LASSO is not strictly  strongly convex, thus failing to meet the requirements of Theorem 3.9 in \cite{Lin2022AlternatingDM}. For Algorithm \ref{alg1} in this paper, further theoretical analysis is required to achieve a linear convergence rate. We plan to conduct this work in the future.

\subsection{Parallel ADMM for ALASSO-PLR with Manual Labels}
In this paper, we primarily focus on large-scale data. This data is so massive that it is difficult to store on a single machine, or it is collected through multi-party synchronization. Therefore, in this subsection, the data is stored in a distributed manner.  Generally, designing algorithms for distributed parallel processing necessitates a central machine and multiple local sub-machines. Suppose that 
\begin{equation}
\begin{aligned}\label{mdata} 
\bm X =(\bm X_1^\top, \bm X_2^\top, \dots, \bm X_G^\top)^\top \ \ \text{and}  \ \  \bm S = (\bm S_1^\top, \bm S_2^\top, \dots, \bm S_G^\top)^\top
\end{aligned}    
\end{equation}
are distributedly stored across $G$ local sub-machines.  
Here $\bm S_g \in \mathbb{R}^{n_g}$,  $\bm{X}_g \in \mathbb{R}^{n_g \times p}$ and $\sum \limits_{g=1}^{G} n_{g} =n$.
The central machine takes in information from local sub-machines, consolidates it, and sends it back to relevant locals. The decomposition in (\ref{mdata}) makes the algorithm naturally parallel across multiple local sub-machines. Each local sub-machine can independently solve its allotted subproblem. Then, the central and local sub-machines communicate and coordinate solutions via variable updates. This distributed storage allows for parallel processing and efficient handling of large-scale data.  For a more detailed discussion, please refer to  Chapters 7 and 8 in \cite{Boyd2010DistributedOA}.

By setting $\bm r_g = \boldsymbol X_g \boldsymbol \beta$ for $g = 1, 2, \ldots, G$, the ALASSO-PLR algorithm can be reformulated as
\begin{align}\label{pcplr} \notag
\min_{\bm{\beta},\bm{r}_g} \ & \sum_{g=1}^G \sum_{i=1}^{n_g} \Big[ -S_i r_i +  m \log\left(1  + \exp({r_i})\right)  \Big] +\lambda \|\bm w \odot \bm \beta \|_1,\\
\text{s.t.} \ \ & \bm r_g = \boldsymbol X_g \boldsymbol \beta, \  g =1,2,\dots,G.
\end{align}
It is not difficult to verify that (\ref{cplr}) and (\ref{pcplr}) are equivalent. The augmented Lagrangian of (\ref{pcplr}) is given by
\begin{equation}
\begin{aligned}\label{alglasso} 
L_{\mu}(\bm \beta, \bm r_g, \bm u_g) & = \sum_{g=1}^G \sum_{i=1}^{n_g} \Big[ -S_i r_i +  m \log\left(1  + \exp({r_i})\right)  \Big] \\
& + \lambda \|\bm w \odot \bm \beta \|_1 - \sum_{g=1}^G \bm u_g^\top ( \boldsymbol X_g \boldsymbol \beta -  \bm r_g ) + \frac{\mu}{2} \sum\limits_{g=1}^{G} \|  \bm X_g \bm \beta - \bm r_g  \|_2^2,
\end{aligned} 
\end{equation}
where  $\bm u_g \in \mathbb{R}^{n_g}$ is the  dual variable and $(\bm u_1^\top, \bm u_2^\top,\dots,\bm u_G^\top)^\top = \bm u$. Similar to \eqref{twoupadmm}, the update of the  variables in the iterative process (with linearization) can be written as 
\begin{equation}\label{pupadmm}
\left\{ \begin{aligned}
\bm \beta^{k+1}  \leftarrow & \mathop {\arg \min }\limits_{\bm \beta} \left\{ L_\mu (\bm \beta, \bm r_g^k,  \bm u_g^k) + \frac{1}{2} \| \bm \beta - \bm \beta^k   \|^2_{{\bm E}_{1}} \right \};\\
\bm r_g^{k+1}  \leftarrow & \mathop {\arg \min }\limits_{\bm r} \left\{ L_\mu (\bm \beta^{k+1}, \bm r_g, \bm u_g^k) + (c_0 - o(1))  \| \bm r_g - \bm r_g^k   \|^2_{\bm I_{n_g}} \right \}, g =1,2,\dots,G; \\
\bm u_g^{k+1}  \leftarrow & \bm{u}_g^{k} - \mu(\bm X_g \bm \beta^{k+1} - \bm r_g^{k+1} ), g =1,2,\dots,G.
\end{aligned} \right.
\end{equation}    
Here, $\boldsymbol{\beta}$ is updated on the central machine, while $\boldsymbol{r}_g$ and $\boldsymbol{u}_g$ are updated on each local sub-machine. 

Looking back at the equation in \eqref{lbeta3}, in a distributed environment, the update of $\bm \beta$ encounters two issues. One issue is that the complete dataset $\boldsymbol{X}$ is unavailable (since $\boldsymbol{X}$ is stored in a distributed manner and we can only access $\bm{X}_g$). The other issue is the computation of the value of $\boldsymbol{\eta}$ based on $\boldsymbol{X}$.  Based on the suggestion in \cite{Wu2023PartitionPA}, we can reformulate the update of $\bm{\beta}$ in \eqref{lbeta2} as
\begin{equation}
\begin{aligned}\label{plbeta}
\bm{\beta}^{k + 1} \leftarrow \underset{\boldsymbol{\beta}}{\mathrm{argmin}}  \left\{ \lambda \left\lVert \bm{w} \odot \bm{\beta} \right\rVert_1 + \frac{\eta}{2} \left\lVert \bm{\beta} - \bm{\beta}^k + \sum_{g = 1}^{G} \bm{\xi}_g^k \right\rVert_2^2 \right\},
\end{aligned}    
\end{equation}
where $\bm{\xi}_g^k = \frac{\mu \bm{X}_g^\top(\bm{X}_g \bm{\beta}^k - \bm{r}_g^k - \bm{u}_g^k / \mu)}{\eta}$. At step $k$, $\bm{\xi}_g^k$ can be calculated in the $g$-th sub-machine and sent to the central machine. It can be obtained as $\bm{\beta}^k$ has been updated on the central machine and delivered to the $g$-th sub-machine where $\bm{X}_g$ is stored, and both $\bm{r}_g$ and $\bm{u}_g$ were updated in the previous step on this sub-machine.  Then, the update of the central machine is as follows
\begin{equation}
\begin{aligned}\label{plbeta2}
\bm \beta^{k+1} \leftarrow  \text{sign}  \left(  \bm \beta^k - \sum_{g = 1}^{G} \bm{\xi}_g^k \right) \odot \left( \left|\bm \beta^k - \sum_{g = 1}^{G} \bm{\xi}_g^k \right| - \frac{\lambda \bm w}{\eta} \right)_+.
\end{aligned}    
\end{equation}
Therefore, the first issue with the $\bm \beta$ update has been resolved. The second issue pertains to the calculation of $\eta$. The calculation of $\eta$ depends on $\bm{X}$. However, each sub-machine only has $\bm{X}_g$. According to the suggestion in \cite{Wu2023PartitionPA}, $\eta_g > \|\mu \bm{X}_g^\top \bm X_g\|$ for each sub-machine can be calculated first. Since $ \bm{X}^\top \bm X =  \sum \limits_{g=1}^G \bm{X}_g^\top \bm X_g$, we have $\sum \limits_{g=1}^G \eta_g > \eta $. Therefore, $\sum \limits_{g=1}^G \eta_g$ can be used to substitute the $\eta$ calculated from the full data. So far, two issues related to the update of $\bm \beta$ in distributed environments have been resolved.

Similar to the derivation of the closed form  solution shown in \eqref{lr3}, we can obtain the update rule for \(\bm r_g\) in a distributed environment,
 \begin{align}\label{plr} \notag
  \bm r_g^{k+1} \leftarrow\left[ \Phi(\bm r_g^k) + \mu \bm I_{n_g}   \right]^{-1}  \left[ \mu  ( \bm X_g \bm \beta^{k+1} - \bm u_g^k/\mu) + \Phi(\bm r_g^k) \bm r_g^k  -  \phi^{'}(\bm r_g^k) + \bm S_g \right],
\end{align}
where $\phi^{'}(\bm r_g^k) = (\phi^{'}(r_1^k),\phi^{'}(r_2^k),\dots,\phi^{'}(r_{n_g}^k))^\top$ and $\Phi(\bm r_g^k)$  is a diagonal matrix, and the diagonal elements are $\{\phi^{''}(r_i^k) +c_0 \}_{i=1}^{n_g}$.

\begin{figure*}
    \centering
    \includegraphics[width=0.9\linewidth]{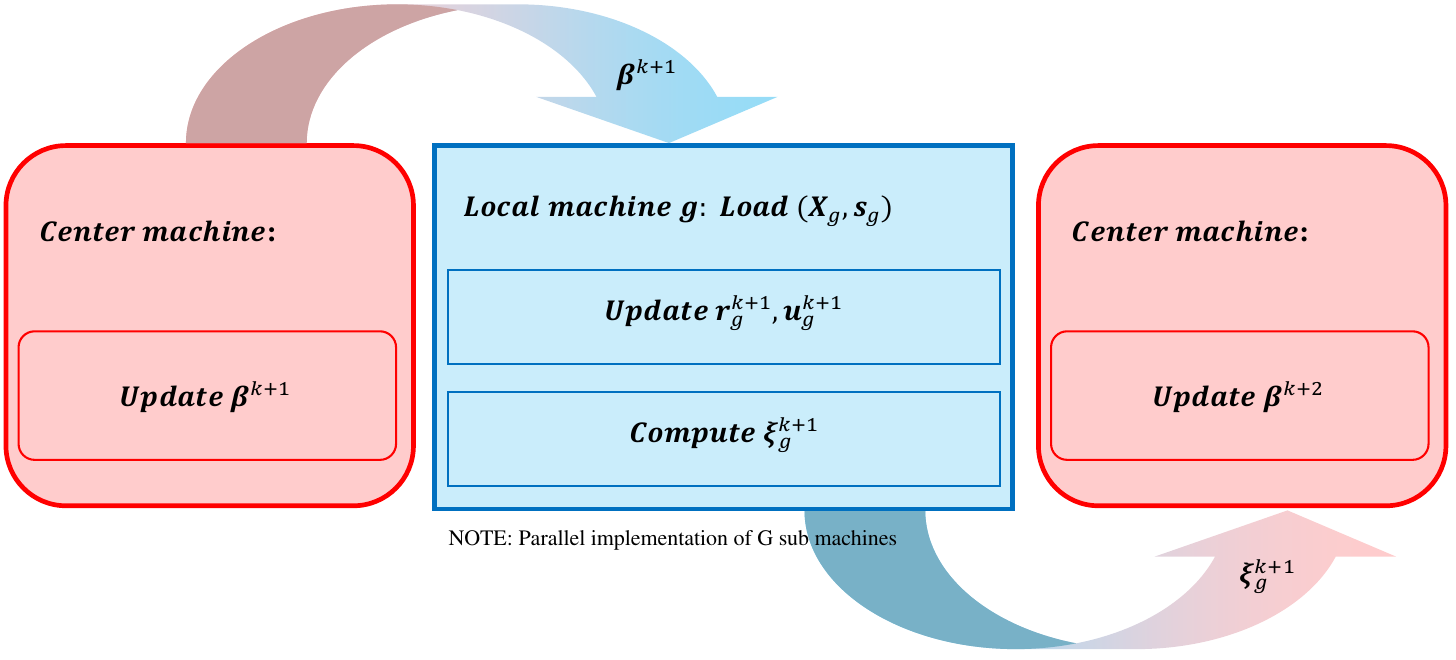}
    \caption{The implementation diagram of parallel algorithm.}
    \label{fig1}
\end{figure*}

We summarize this parallel ADMM algorithm for ALASSO-PLR with noise on Algorithm \ref{alg2}. The implementation diagram of parallel algorithm is shown in Figure \ref{fig1}.
\begin{algorithm}\small
\caption{\small{Parallel ADMM for ALASSO-PLR with noise}}
\label{alg2}
\begin{algorithmic}
\STATE {\textbf{Pre-computation}:  $\eta_g =\|\mu \bm X_g^\top \bm X_g \|$, $g= 1,2,\dots,G$.} 
\STATE {\textbf{Input:}  $\bullet$  central machine:  primal variable $\boldsymbol{\beta}^0$;  $\eta = \sum \limits_{g=1}^{G}\eta_g$;  penalty parameter $\lambda$; and weight vector $\bm w$. \\
\quad \  \ \ \ \  \ \ \  $\bullet$ the $g$-th local machine:  observation data $\boldsymbol{X}_g$; manual label $\bm S_g$; primal variables $\boldsymbol{\beta}^0,\boldsymbol{r}_g^0$;  dual variables $\boldsymbol{u}_g^0$; augmented parameter $\mu$; linearization parameters $\eta$; and $\bm \xi_g^0 =\frac{\mu \bm{X}_g^\top(\bm{X}_g \bm{\beta}^0 - \bm{r}_g^0 - \bm{u}_g^0 / \mu)}{\eta}$.}
\STATE {\textbf{Output:}  the total number of iterations $K$,  $\boldsymbol{\beta}^K, \bm r^K, \bm u^K$. }
\STATE {\textbf{while} not converged \textbf{do}}
\STATE {\ \textbf{Central machine}: \\
\quad 1. Receive $\eta_g$ (only once) and $\bm \xi_g^k$ transmitted by all  sub-machines,\\ 
\quad 2.  Update $\bm \beta^{k+1}$  using (\ref{plbeta2}), \\
\quad 3. Send $\bm \beta^{k+1}$ to the each sub-machines.
}
\STATE {\ \textbf{Local machines}: \ \  for $g =1 ,2, \dots, G$ (in parallel) \\
\quad 1. Receive $\bm \beta^{k+1}$  transmitted by the central machine, \\
\quad 2.  Update $\bm r^{k+1} $ using \eqref{plr},  \\
\quad 3. Update $\bm u_g^{k+1}  \leftarrow  \bm{u}_g^{k} - \mu(\bm X_g \bm \beta^{k+1} - \bm r_g^{k+1} ),$ \\
\quad 4. Compute $\bm{\xi}_g^{k+1} = \frac{\mu \bm{X}_g^\top(\bm{X}_g \bm{\beta}^{k+1} - \bm{r}_g^{k+1} - \bm{u}_g^{k+1} / \mu)}{\eta}$, \\
\quad 5. Send $\bm \xi^{k+1}_g$  to the central machine.
}
\STATE {\textbf{end while}}
\STATE {\textbf{return} solution}.
\end{algorithmic}
\end{algorithm}

Next, we will theoretically elaborate on the equivalence between Algorithm \ref{alg1} and Algorithm \ref{alg2}. In fact, when $G = 1$, the two algorithms are identical. Therefore, the subsequent discussion will center on the case where $G \geq 2$ in Algorithm \ref{alg2}. To facilitate differentiation, let us denote $\left\{ \hat{\bm \beta}^k, \hat{\bm r}^k, \hat{\bm u}^k\right\}$ as the outcomes of the $k$-th iteration of Algorithm \ref{alg1}, and $\left\{ \bar{\bm \beta}^k, \bar{\bm r}^k, \bar{\bm u}^k \right\}$ as the outcomes of the $k$-th iteration of Algorithm \ref{alg2}.
Then, we can arrive at the following conclusion. 
\begin{thm}\label{TH1}
If we employ the same initial iteration variables $\{ \hat{\bm \beta}^0, \hat{\bm r}^0, \hat{\bm u}^0 \} = \{ \bar{\bm \beta}^0, \bar{\bm r}^0, \bar{\bm u}^0 \}$ and the same parameter $\eta$, the iterative solutions derived from these two algorithms are, in fact, identical. That is to say,
\begin{align}
\left\{ \hat{\bm \beta}^k, \hat{\bm r}^k, \hat{\bm d}^k\right\} = \left\{ \bar{\bm \beta}^k, \bar{\bm r}^k, \bar{\bm d}^k \right\}, \quad \text{for all } k.
\end{align}
\end{thm}

The proof of Theorem \ref{TH1} can be found in supplementary material 1.4. The conclusion of Theorem \ref{TH1} indicates that when using the same initial values, the solutions of  Algorithm \ref{alg1} and Algorithm \ref{alg2}  will be the same, which is what Wu et al. \cite{Wu2023PartitionPA} called partition insensitivity. In parallel computing, insensitivity to data partitioning can enhance computational efficiency (by reducing communication overhead and increasing parallelism), improve system stability (by minimizing the risk of load imbalance and enhancing fault tolerance), and simplify algorithm design and implementation (by lowering design complexity and facilitating code maintenance).  However, in reality,$\sum \limits_{g=1}^G \eta_g$ (Algorithm \ref{alg2})  is usually greater than $\eta$ (Algorithm \ref{alg1}) in applications, but this does not affect the convergence of Algorithm \ref{alg2} with Theorem \ref{th4}. Simply replace $\eta$  in $\bm E_1$  with $\sum \limits_{g=1}^G \eta_g$ in Theorem \ref{th4}.
\section{Numerical Experiments}\label{sec5}
In this section, we conduct comprehensive evaluations of the proposed method, theoretical framework, and algorithms using both synthetic and real-world datasets. All computational experiments were performed in R on a workstation equipped with an AMD Ryzen 9 7950X 16-Core processor (4.50 GHz) and 32 GB RAM. In the process of parameter tuning, we utilize the modified HBIC criterion presented in \cite{Wu2023PartitionPA}. Moreover, the termination conditions of the algorithm adhere to those described in \cite{Boyd2010DistributedOA}. 
\subsection{Synthetic Data Experiment}\label{appE}
We illustrate the methodology using the penalized logistic regression model of \cite{Zou2006TheAL}. In this example, we simulated 100 datasets consisting of $n$ observations from the model \( y \sim  \text{Binomial}\{p(\bm{x}^\top \bm{\beta}^*)\} \), where \( p(u) = \exp(u)/(1 + \exp(u)) \) and \( \bm{\beta}^* = (3, 0, 0, 1.5, 0, 0, 7, 0, 0) \). The components of \( \bm{x} \) are standard normal, where the correlation between \( x_i \) and \( x_j \) is \( \rho = 0.75 \).

First,  we present the performance of the PLR with manual labels for different values of $m$ and $\alpha_0$. Figure \ref{fig2} shows how the coefficient estimation values change with $m$ and $\alpha_0$, where the estimations are averaged over 100 simulations. In the three sub-figures in the upper part of Figure \ref{fig2}, $\alpha_0$ is set to 1. From these figures, we can observe that as $m$ increases, the results of coefficient estimation get better, especially for the locus with a value of 7. In the three sub-figures in the lower half of Figure \ref{fig2}, $m$ is set to 20. From these figures, we can observe that as $\alpha_0$ increases, the coefficient estimation performance improves, particularly at the positions with values of 1.5 and 7. 
\begin{figure}
    \centering
    \includegraphics[width=0.8\linewidth]{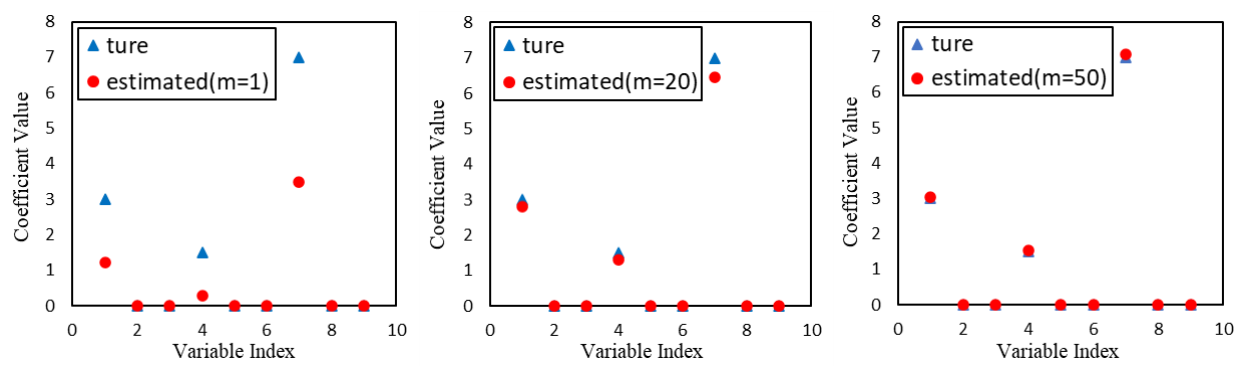}
    \includegraphics[width=0.8\linewidth]{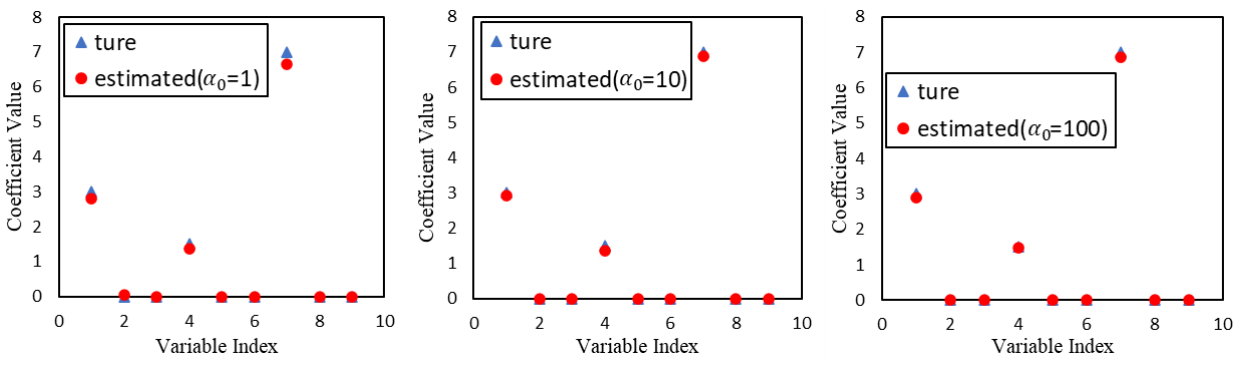}
    \caption{\footnotesize{Schematic diagram of coefficient estimation values changing with $m$ and $\alpha_0$}}
    \label{fig2}
\end{figure}

\begin{table}[H]
    \caption{{Simulated asymptotic relative efficiency (ARE) for different $m, n$ and $\alpha_0$.}}
    \centering
    \renewcommand{\arraystretch}{1.5}
     \resizebox{0.7\textwidth}{!}{
\begin{tabular}{cccccccc}
\hline
$m$ & $\alpha_0$ & ARE & $n = 100$ & $n = 200$ & $n = 500$ & $n = 1000$ & $n = 2000$\\
\hline
\multirow{4}{*}{5} & 1 & 1.67 & 1.55 (0.08) & 1.61 (0.08) & 1.72 (0.08) & 1.70 (0.07) & 1.68 (0.07)\\
  & 10 & 3.67 & 3.21 (0.16) & 3.26 (0.17) & 3.35 (0.19) & 3.54 (0.18)& 3.69 (0.11) \\
  & 100 & 4.81 & 4.27 (0.23) & 4.38 (0.24) & 4.45 (0.25) & 4.61 (0.24)& 4.79 (0.25) \\
  & 1000 & 4.98 & 4.33 (0.25) & 4.41 (0.26) & 4.53 (0.29) & 5.13 (0.31) & 4.95 (0.28)\\
\hline
\multirow{4}{*}{10} & 1 & 1.82 & 1.61 (0.09) & 1.66 (0.09) & 1.69 (0.10) & 1.72 (0.11) & 1.79 (0.12) \\
   & 10 & 5.50 & 5.09 (0.21) & 5.13 (0.22) & 5.24 (0.25) & 5.31 (0.24)& 5.40 (0.26) \\
   & 100 & 9.18 & 8.55 (0.39) & 8.62 (0.43) & 8.71 (0.44) & 8.89 (0.46) & 9.11 (0.49)\\
   & 1000 & 9.91 & 8.96 (0.42) & 9.08 (0.45) & 9.46 (0.48) & 9.63 (0.50)& 9.75 (0.52) \\
\hline
\multirow{4}{*}{50} & 1 & 1.96 & 1.73 (0.10) & 1.78 (0.11) & 1.82 (0.12) & 1.89 (0.13) & 1.93 (0.14)\\
   & 10 & 9.17 & 8.30 (0.35) & 8.42 (0.38) & 8.66 (0.41) & 8.89 (0.42) & 9.08 (0.43)\\
   & 100 & 33.67 & 28.91 (1.37) & 29.21 (1.45) & 30.24 (1.58) & 35.58 (1.85) & 34.07 (1.65)\\
   & 1000 & 47.67 & 40.33 (1.72) & 42.74 (1.78) & 44.17 (1.93) & 45.27 (2.04) & 48.01 (2.08)\\
\hline
\end{tabular}}
\label{tab3}
\end{table}

In Table \ref{tab3}, we compare the simulated asymptotic relative efficiency (ARE) for different values of $m$, $n$, and $\alpha_0$. The ARE values presented in the table represent the theoretical values. The estimated ARE values shown in the table are the average results of 100 simulations, with the standard deviation provided in parentheses.   In this table, as the values of $m$, $n$ and $\alpha_0$ increase, the estimated ARE will perform better and better. This is consistent with the conclusion of consistent estimators in statistics and the theoretical results revealed in this paper.

For large-scale data, we set $n$ to be $1,000,000$ and $2,000,000$ respectively, and conduct numerical simulations in a parallel computing environment where $G$ takes the values of $1, 5, 10$, and $20$. To verify the conclusion of Theorem \ref{TH1}, the algorithms proposed in this paper all use $\eta = \|\mu \bm X^\top \bm X\| + 0.01$.
Most of the existing parallel algorithms designed for distributed storage data are inspired by \cite{Boyd2010DistributedOA}. They utilize consensus constraints to attain algorithm parallelism. For more details, refer to \cite{Ding2024AGA} and the references cited therein. Here, we will compare the partition insensitive parallel ADMM algorithm (PIPADMM) proposed in this paper with the consensus structure parallel ADMM algorithm (CPADMM) provided in \cite{Ding2024AGA} for solving  PLR. Table \ref{tab22} presents the mean results of 100 independent simulations, with the standard deviation in parentheses. The following is an explanation of the various indicators in the table. FP is a false positive (zero coefficient is estimated to be non-zero), FN is a false negative (non-zero coefficient is estimated to be zero), AE represents the absolute error of estimation ($\| \hat{\bm \beta} - \bm \beta^* \|_1$), Ite is the number of iterations, and Time represents the computer running time. 

One of the most prominent manifestations is that as the number of parallel machines increases, in addition to a significant reduction in the computation time of PIPADMM, all other indicators remain unchanged. This is also what Theorem \ref{TH1} in this paper refers to as partition insensitivity, that is, regardless of how the samples are partitioned, the solution of each step of the PIPADMM algorithm remains the same. The reduction in computation time is due to the parallel update of \(\bm r\). In contrast, for the CPADMM algorithm, as the value of \(G\) increases, the quality of the solution (FP, AE) deteriorates, and the number of iterations (Ite) increases. Similarly, due to parallel computing, the computation time significantly decreases. The decline in solution quality and the increase in the number of iterations are due to the inherent flaws of the parallel ADMM algorithm with a consensus structure. As the number of machines increases, more consensus constraints are constructed, leading to more variables to be updated in the ADMM algorithm. This results in lower solution accuracy and slower algorithm convergence, which is also observed in \cite{Wu2025AUC} when solving  linear regression models. In the comparison between the two parallel ADMM algorithms, PIPADMM outperforms in all metrics. Due to its theoretical property of partition insensitivity, PIPADMM will perform even better than CPADMM as the value of \(G\) increases.  Overall, the experimental results in Table \ref{tab22} indicate that our proposed PIPADMM algorithm has more stable, accurate, and faster computational performance in computing PLR compared to CPADMM algorithms in parallel computing environments.

\begin{table}[H]
\caption{{Comparison results of  parallel ADMM algorithms under  large-scale data.}}
\label{tab22}
\centering
\renewcommand{\arraystretch}{1.5}
\resizebox{0.9\linewidth}{!}{
\begin{tabular}{lcccccccccc}
\hline
 &  \multicolumn{5}{c}{$n = 1,000,000$ \ (PIPADMM)}  &\multicolumn{5}{c}{$n = 1,000,000$ \ (CPADMM)}  \\ 
\cmidrule(lr){2-6}\cmidrule(lr){7-11}
G &  FP &  FN &  AE   &  Ite   &  Time  &  FP &  FN &  AE   &  Ite   &  Time \\ \hline
1       & 0 & 0 & {0.024(0.005)} & 92.3(8.52)& 20.4(2.57) & 0.17 & 0 & 0.075(0.012) & 121.2(11.3)  & 27.4(3.61)  \\
5   & 0 & 0 & {0.024(0.005)} & 92.3(8.52)& 5.40(0.85) & 0.21 & 0 & 0.106(0.020) & 167.5(18.9) & 7.41(1.26)  \\
10     & 0 & 0 & {0.024(0.005)} & 92.3(8.52)  & 2.92(0.50) & 0.26 & 0 & 0.162(0.041) & 205.1(27.9)  & 4.76(0.85) \\
20      & 0 & 0 & {0.024(0.005)} & 92.3(8.52) & 1.42(0.27) & 0.31 & 0 & 0.199(0.066) & 226.3(31.8) & 2.98(0.51) \\
\midrule[1pt]
  & \multicolumn{5}{c}{$n=2,000,000$ \ (PIPADMM)} & \multicolumn{5}{c}{$n=2,000,000$  (CPADMM)} \\ 
\cmidrule(lr){2-6}\cmidrule(lr){7-11}
 G &  FP &  FN &  AE   &  Ite     &  Time  &  FP &  FN &  AE   &  Ite     &  Time\\ \hline
1  & 0 & 0 & {0.019(0.003)} & 97.2(9.09)  & 41.3(3.98) &0.12 & 0 & 0.162(0.023) & 145.8(12.7)   & 59.7(7.92)\\
5 & 0 & 0 & {0.019(0.003)} & 97.2(9.09)  & 10.8(1.05)  &0.14 & 0 & 0.205(0.034) & 159.7(15.2)  & 38.2(4.76) \\
10  & 0 & 0 & {0.019(0.003)} & 97.2(9.09)  & 6.29(0.53) &0.17 & 0 & 0.396(0.045) & 190.4(20.5)   & 22.4(2.97)\\
20   &{0}& 0 & {0.019(0.003)} & 97.2(9.09) & 3.92(0.29) &0.23 & 0 & 0.471(0.052) &213.5(27.2)  & 16.1(1.72) \\ \hline
\end{tabular}}
\end{table}
\subsection{Real Data Experiments}\label{sec52}
In this section, we validate that manual labeling can enhance the variable selection ability of PLR on datasets in Table \ref{tab1}. These datasets is available at \url{https://www.csie.ntu.edu.tw/~cjlin/libsvmtools/datasets/.} Due to space limitations, we only present the numerical results of the w8a dataset with the largest data size in the table. The numerical results of other datasets  can be found in supplementary material 1.5.
\begin{table}[H]
    \centering
    \renewcommand{\arraystretch}{1.5}
    \caption{{Basic information of datasets.}}
    \resizebox{0.4\textwidth}{!}{
    \begin{tabular}{lcccc}
     \toprule[\heavyrulewidth] 
     & ijcmn1 & phishing & a9a & w8a \\
    \midrule
    Features (d) & 22 & 68 & 123 & 300 \\
    Samples (n)  & 49990 & 11055 & 32561 & 49749 \\
    \bottomrule[\heavyrulewidth] 
    \end{tabular}}
    \label{tab1}
\end{table}

We randomly partitioned the original dataset comprising 49,749 observations into a training set of 40,000 observations and a test set of 9,749 observations. Since our algorithm exhibits partition insensitivity, to save computation time and account for the device's memory, we set \(G = 20\) in the parallel computing environment. Table \ref{tab2} presents the simulated relative efficiency of PLR.  The true coefficient $\bm{\beta}^*$, which serves to generate a posterior probability for manual labeling, is unknown. Here, it can be estimated using the entire sample data.  The simulated ARE values tend to increase as $\alpha_0$ and  $m$ increase, which is consistent with the conclusion of Theorem \ref{th3}.
It is observed that the simulated relative efficiencies are much lower than their corresponding theoretical values.   This discrepancy might stem from two factors. Firstly, there could be an underestimation of the conditional error rate $\text{err}(\hat{\bm \beta}; \bm \beta^*)$ in \eqref{are}. Secondly, the true conditional error rate $\text{err}({\bm \beta}^*; \bm \beta^*)$ may not be precisely estimated using the limited test set.  
In the simulations presented in Section \ref{appE}, we are able to compute the exact $\text{err}(\hat{\bm \beta}; \bm \beta^*)$ and $\text{err}({\bm \beta}^*; \bm \beta^*)$. This explains why the estimated ARE is closer to the theoretical ARE in Section \ref{appE}. 

\begin{table}[H]
    \centering
    \renewcommand{\arraystretch}{1.5}
    \caption{{The asymptotic relative efficiency (ARE) on the \textit{w8a} dataset. The theoretical ARE of PLR, as stated in Theorem \ref{th3}, is provided in parentheses.}}
    \label{tab2}
    \resizebox{0.65\textwidth}{!}{
    \begin{tabular}{*{6}{c}}
        \toprule
        $\alpha_0$ & $m=5$ & $m=10$ & $m=20$ & $m=50$ & $m=100$ \\
        \midrule
        $1$    & 1.32(1.67) & 1.39(1.82) & 1.45(1.90) & 1.51(1.96) &1.59 (1.98) \\
        $10$   & 2.56 (3.67) & 4.10(5.50) & 4.96(7.33) & 6.78(9.17) & 8.42(10.00) \\
        $100$  & 3.07(4.81) & 5.36(9.18) & 8.77(16.83) & 15.22(33.67) & 25.69(50.50) \\
        $500$  & 3.14(4.96) & 6.09(9.82) & 9.25(19.27) & 19.67(45.55) & 39.58(83.50) \\
        $1000$ & 3.21(4.98) & 6.87(9.91) & 9.93(19.63) & 26.32(47.67) & 45.36(91.00) \\
        \bottomrule
    \end{tabular}}
\end{table}

When examining the theoretical error rates, we can once again observe that for larger values of $\alpha_0$ or $m$, the other parameter  has a more pronounced influence. 
Specifically, when $\alpha_0 = 1$ and $m = 100$, the theoretical 
ARE is 1.98. This value is quite close to the limiting ARE value of 2 as $m \to \infty$. On the other hand, when $\alpha_0 = 1000$ and $m = 100$, the theoretical ARE is 91.00, which is much closer to the limiting ARE value of 100 as $\alpha_0 \to \infty$.

\section{Conclusion and Future Work}\label{sec6}
This paper theoretically proves that when label noise is only related to classification difficulty, it is useful for variable selection in binary classification PLR. This benefit will be reflected in more accurate estimates of important coefficients compared to using truth labels. In addition, we also propose a new parallel algorithm for solving the large-scale PLR with manually generated labels. This parallel algorithm is insensitive to data partitioning, meaning that the solution of the algorithm remains unchanged regardless of different data partitions.

However, the theoretical results in this paper are only applicable when \(p < n\), that is, when the feature dimension is smaller than the number of samples. Extending the beneficial noise concept to high-dimensional cases (\(p>n\)) or even ultra-high-dimensional cases (\(p\gg n\)) is a challenging task.  Future research can be extended in several directions: expanding to multi-class classifier \cite{Friedman2010RegularizationPF} and multi-view classifier \cite{Chen2025MultiSV}  to see if the beneficial effect of label noise persists;  optimizing the parallel algorithm for better computational efficiency and scalability using techniques like GPU acceleration; incorporating temporal and spatial information into the PLR model and algorithm to handle data with such characteristics.

%
%
%
%

\appendix
\section{Proof of Throrem 1}\label{appA}
We first prove the asymptotic normality part. Recall $\phi({\bm{x}^\top \bm{\beta}}) =  m \log\left(1  + \exp({{\bm{x}^\top \bm{\beta}}})\right)$ and let \( \bm \beta = \bm \beta^* + \frac{\bm u}{\sqrt{n}} \) with $\bm u = (u_1, u_2, \dots,u_d)^\top$. Define

\begin{equation}
    \begin{aligned}\notag
     \Gamma_n(\mathbf{u})  = \sum_{i=1}^n \left( -S_i \left( \bm x_i^\top  \left( \bm \beta^* + \frac{\bm u}{\sqrt{n}} \right) \right) + \phi \left( \bm x_i^\top \left( \bm \beta^* + \frac{\bm u}{\sqrt{n}} \right) \right) \right) + \lambda \sum_{j=1}^d w_j \left|  \beta_j^* + \frac{u_j}{\sqrt{n}} \right|.   
    \end{aligned}
\end{equation}

Let \( \tilde{\bm{u}}^{(n)} = \arg \min\limits_{\bm{u}} \Gamma_n(\bm{u}) \); then \( \tilde{\bm{u}}^{(n)} = \sqrt{n} (\tilde{\bm \beta} - \bm \beta^*) \), where \( \tilde{\bm \beta}\) is the PLR estimator with manual labels in (3.2). 

Using the Taylor expansion, we have \(\Gamma_n(\bm{u}) - \Gamma_n(\bm{0}) = H^{(n)}(\bm{u})\), where
\[
H^{(n)}(\bm{u}) \equiv A_1^{(n)} + A_2^{(n)} + A_3^{(n)} + A_4^{(n)},
\]
with
\[
A_1^{(n)} = -\sum_{i=1}^n \left[ S_i - \phi' \left( \bm{x}_i^\top \bm{\beta}^* \right) \right] \frac{\bm{x}_i^\top \bm{u}}{\sqrt{n}},
\]
\[
A_2^{(n)} = \sum_{i=1}^n \frac{1}{2} \phi'' \left( \bm{x}_i^\top \bm{\beta}^* \right) \bm{u}^\top \frac{\left( \bm{x}_i \bm{x}_i^\top \right)}{n} \bm{u},
\]
\[
A_3^{(n)} = \frac{\lambda}{\sqrt{n}} \sum_{j=1}^d w_j \sqrt{n} \left( \left| \beta_j^* + \frac{u_j}{\sqrt{n}} \right| - \left| \beta_j^* \right| \right),
\]
and
\[
A_4^{(n)} = n^{-3/2} \sum_{i=1}^n \frac{1}{6} \phi{'''} \left( \bm{x}_i^\top \tilde{\bm{\beta}}_* \right) \left( \bm{x}_i^\top \bm{u} \right)^3,
\]
where \(\tilde{\bm{\beta}}_*\) is between \(\bm{\beta}^*\) and \(\bm{\beta}^* + \frac{\bm{u}}{\sqrt{n}}\). 

We analyze the asymptotic limit of each term. By the proof of Theorem 1 of \cite{Ahfock2021HarmlessLN}, we observe that
\begin{align}\label{exp}
    \mathbb E_{S_i, \bm{x}_i} \left( \left[ S_i - \phi' \left( \bm{x}_i^\top \bm{\beta}^* \right) \right]  \bm{x}_i^\top   \right) = \bm 0 \tag{S1.1}
\end{align} and 
\begin{align}\label{var}
\text{Var}_{S_i, \bm{x}_i} \left( \left[ S_i - \phi' \left( \bm{x}_i^\top \bm{\beta}^* \right) \right]  \bm{x}_i^\top  \right)  =  \frac{ m+\alpha_0 }{1+\alpha_0}m I(\bm{\beta}^*) .    \tag{S1.2}
\end{align}
It follows that 
\[
\mathbb E_{S_i, \bm{x}_i} \left( \left[ S_i - \phi' \left( \bm{x}_i^\top \bm{\beta}^* \right) \right] \left( \bm{x}_i^\top \bm{u} \right) \right) = 0
\]
and
\begin{equation}
\tag{S1.3}
    \begin{aligned}
     \text{Var}_{S_i, \bm{x}_i} \left( \left[ S_i - \phi' \left( \bm{x}_i^\top \bm{\beta}^* \right) \right] \left( \bm{x}_i^\top \bm{u} \right) \right) = \bm{u}^\top  \frac{ m+\alpha_0 }{1+\alpha_0}m I(\bm{\beta}^*) \bm{u}. 
    \end{aligned}
\end{equation}

Then the central limit theorem says that 
\begin{align}\label{A1}
    A_1^{(n)} \xrightarrow{d} \bm{u}^\top N(0, \frac{ m+\alpha_0 }{1+\alpha_0}m I(\bm{\beta}^*)). \tag{S1.4}
\end{align}\label{A21}
 For the second term \(A_2^{(n)}\),  from (2.2), we observe that
\begin{align}\label{S15}
    \sum_{i=1}^n \phi'' \left( \bm{x}_i^\top \bm{\beta}^* \right) \frac{\left( \bm{x}_i \bm{x}_i^\top \right)}{n} \xrightarrow{p}  m I(\bm{\beta}^*), \tag{S1.5}
\end{align}
where $\xrightarrow{p}$ represents convergence in probability. Thus, 
\begin{align}\label{A2}
    A_2^{(n)} \xrightarrow{p}  \frac{m}{2}  \bm{u}^\top  I(\bm{\beta}^*) \bm{u}. \tag{S1.6}
\end{align}
The limiting behavior of the third term is discussed in the proof of Theorem 2 in \cite{Zou2006TheAL}. We summarize the results as follows:
\begin{align}\label{A3}
    \frac{\lambda}{\sqrt{n}} {w}_j \sqrt{n} \left( \left| \beta_j^* + \frac{u_j}{\sqrt{n}} \right| - \left| \beta_j^* \right| \right) \xrightarrow{p} \tag{S1.7}
\begin{cases}
0 & \text{if } \beta_j^* \neq 0 \\
0 & \text{if } \beta_j^* = 0 \text{ and } u_j = 0 \\
\infty & \text{if } \beta_j^* = 0 \text{ and } u_j \neq 0.
\end{cases}
\end{align}
By the regularity condition (A2), the fourth term can be bounded as
\begin{align}\label{A4} 
    6\sqrt{n} A_4^{(n)} \leq & \sum_{i=1}^n \frac{1}{n} M(\bm{x}_i^\top) |\bm{x}_i^\top \bm{u}|^3 \xrightarrow{p}  E[M(\bm{x}) |\bm{x}^\top \bm{u}|^3] < \infty. \tag{S1.8}
\end{align}
Thus, based on the results from \eqref{A1}, \eqref{A2}, \eqref{A3} and \eqref{A4}, by invoking Slutsky’s theorem, we can infer that \( H^{(n)}(\bm{u}) \xrightarrow{d} H(\bm{u}) \) for every \( \bm{u} \), where

\[
H(\bm{u}) = 
\begin{cases} 
m \bm{u}_{\mathcal{A}}^\top I_{\mathcal{A}} \bm{u}_{\mathcal{A}} - 2\bm{u}_{\mathcal{A}}^\top \bm{W}_{\mathcal{A}} & u_j = 0 \text{ for all } j \notin \mathcal{A} \\
\infty & \text{otherwise},
\end{cases}
\]
and \( \bm{W}_{\mathcal{A}} \sim N\left(0, \frac{m + \alpha_0}{1 + \alpha_0}m I_{\mathcal{A}} \right) \). \( H^{(n)} \) is convex and the unique minimum of \( H \) is \( (I_{\mathcal{A}}^{-1} \bm{W}_{\mathcal{A}}/m, \bm{0})^T \). Then we have
\begin{align}\label{asy}
    \tilde{\bm{u}}_\mathcal{A} ^{(n)} \xrightarrow{d} N\{0, \frac{m+ \alpha_0}{m(1+  \alpha_0)}I_\mathcal{A}^{-1}\} \quad \text{and} \quad \tilde{\bm{u}}_{\mathcal{A}^c}^{(n)}  \xrightarrow{d} \bm{0}. \tag{S1.9}
\end{align}
To sum up, the asymptotic normality part is proven.

Now we show the consistency part. \(\forall j \in \mathcal{A}\), the asymptotic normality indicates that \(\Pr(j \in \tilde{\mathcal{A}}) \to 1\). Then it suffices to show that \(\forall j' \notin \mathcal{A}\), \(\Pr(j' \in \tilde{\mathcal{A}}) \to 0\). Consider the event \(j' \in \mathcal{A}\). By the KKT optimality conditions, we must have
\begin{align}\label{con}
   \sum_{i=1}^n x_{ij'} \left( S_i - \phi'\left( \bm{x}_i^\top \tilde{\bm{\beta}} \right) \right) = \lambda w_j; \tag{S1.10}
\end{align}
thus \[\Pr(j' \in \tilde{\mathcal{A}}) \leq \Pr \left( \sum\limits_{i=1}^n x_{ij'} \left(S_i - \phi'\left( \bm{x}_i^\top \tilde{\bm{\beta}} \right) \right) = \lambda w_j \right).\] Note that
\[
\sum_{i=1}^n x_{ij'} \left(S_i - \phi'\left( \bm{x}_i^\top \tilde{\bm{\beta}} \right)  \right)/ \sqrt{n} = {B}_1^{(n)} + {B}_2^{(n)} + {B}_3^{(n)},
\]
with
\[
B_1^{(n)} = \sum_{i=1}^n x_{ij'} \left(S_i - \phi'\left( \bm{x}_i^\top {\bm{\beta}^*} \right)  \right)/ \sqrt{n},
\]
\[
B_2^{(n)} = \left( \frac{1}{n} \sum_{i=1}^n x_{ij'} \phi'' \left( \bm{x}_i^\top {\bm{\beta}}^* \right) \bm{x}_i^\top \right) \sqrt{n} \left( \tilde{\bm{\beta}} - {\bm{\beta}}^* \right),
\]
and
\[
B_3^{(n)} = \frac{1}{n} \sum_{i=1}^n x_{ij'} \phi{'''} \left( \bm{x}_i^\top \tilde{\bm{\beta}}^{*} \right) \left( \bm{x}_i^\top \sqrt{n} \left( \tilde{\bm{\beta}} - \hat{\bm{\beta}}^* \right) \right)^2 / \sqrt{n},
\]
where \(\tilde{\bm{\beta}}^{*}\) is between \(\tilde{\bm{\beta}}\) and \({\bm{\beta}}^*\). From \eqref{exp} and \eqref{var}, we know that 
\begin{align*}
   B_1^{(n)} \xrightarrow{d}  N(0, \frac{ m+\alpha_0 }{1+\alpha_0}m I(\bm{\beta}^*)). 
\end{align*}
It follows from \eqref{S15} that
\[
\frac{1}{n} \sum_{i=1}^n x_{ij'} \phi'' \left( \bm{x}_i^\top {\bm{\beta}}^* \right) \times \bm{x}_i^\top \xrightarrow{p}  m I_{j^{'}}(\bm{\beta}^*), 
\]
where $I_{j^{'}}(\bm{\beta}^*)$ is the $j^{'}$-th row of $I(\bm{\beta}^*)$. Combining \eqref{asy}, it can be concluded that  $B_2^{(n)}$ converges to some normal random variable.  It follows from the regularity condition  (A2) and \eqref{asy} that \( B_3^{(n)} = O_p(1/\sqrt{n}) \). Meanwhile, since $\boldsymbol{\beta}^{'}$ is a $\sqrt{n}$-consistent estimator of $\boldsymbol{\beta}^*$,  we have $\sqrt{n} {\beta}_j' \xrightarrow{p}  O(1) $. Then,
\[
\frac{\lambda w_j}{\sqrt{n}} = \frac{\lambda}{\sqrt{n}} n^{\gamma/2} \frac{1}{\left| \sqrt{n} {\beta}_j' \right|^\gamma} \xrightarrow{p} \infty.
\]
By considering the left and right sides of the \eqref{con}, it can be concluded that \(\Pr(j' \in \tilde{\mathcal{A}}) \to 0\). This completes the proof of Theorem 1.

\section{Proof of Throrem 2}\label{appB}
From the conclusions of Section 2.1.2 and Theorem 1, it can be known that when \(n \to \infty\),  $\bm \beta^* = ({\bm \beta^*_{\mathcal{A}}}^\top, \bm 0^\top)^\top$, $\hat{\bm \beta} = (\hat{\bm \beta}_{\mathcal{A}}^\top, \bm 0^\top)^\top$ and $\tilde{\bm \beta} = (\tilde{\bm \beta}_{\mathcal{A}}^\top, \bm 0^\top)^\top$. This also means that  $\text{err}(\hat{\bm \beta }; \bm \beta^*) =  \sum \limits_{i =1}^n \Pr(\hat {y}_i \ne y_i \mid  \bm \beta^* ) = \sum \limits_{i =1}^n \Pr(\hat {y}_i \ne y_i \mid  \bm \beta^*_{\mathcal{A}} ) = \text{err}(\hat{\bm \beta}_{\mathcal{A}} ; \bm \beta^*_{\mathcal{A}}) $, where   $\hat {y}_i \sim \text{Binomial}(1, \frac{1}{1 +\exp(- \hat{\bm \beta}_{\mathcal{A}}^\top  {\bm x_i}_{\mathcal{A}}) })$. Similarly, the following conclusion also holds true, 
$\text{err}(\tilde{\bm \beta }; \bm \beta^*) = \text{err}(\tilde{\bm \beta}_{\mathcal{A}} ; \bm \beta^*_{\mathcal{A}})$. Before proving the theorem, a lemma about the asymptotic excess errors is needed.
\begin{lem}\label{lem1}
    The asymptotic excess errors can be expanded as
\begin{align}\notag
    \lim_{n \to \infty} n \left[ \mathbb E\left( \text{err}(\hat{\bm \beta}_{\mathcal{A}} ; \bm \beta^*_{\mathcal{A}}) \right) -  \text{err}(\bm \beta^*_{\mathcal{A}}; \bm \beta^*_{\mathcal{A}}) \right] = \operatorname{trace}\{ J(\bm \beta^*_{\mathcal{A}}) I^{-1}_{\mathcal{A}} \}
\end{align} 
and
\begin{align}\notag
    & \lim_{n \to \infty} n \left[ \mathbb E\left( \text{err}(\tilde{\bm \beta}_{\mathcal{A}} ; \bm \beta^*_{\mathcal{A}}) \right) -  \text{err}(\bm \beta^*_{\mathcal{A}} ; \bm \beta^*_{\mathcal{A}}) \right] = \frac{m+\alpha_0 }{m(1+\alpha_0)} \operatorname{trace}\{ J(\bm \beta^*_{\mathcal{A}}) I^{-1}_{\mathcal{A}} \},
\end{align}
where 
\[
J(\beta) = \frac{1}{2} \left[\operatorname{err}^{''}( \bm {\beta}^{'}; \bm \beta^*) \right]_{\bm {\beta}^{'}  = \bm \beta^*}.
\]
\end{lem}
\begin{proof}
First, consider the Taylor expansion of the error function \(\text{err}(\hat{\bm \beta}_{\mathcal{A}} ; \bm \beta^*_{\mathcal{A}}) \) around the true parameter \(\bm \beta^*_{\mathcal{A}}\):
\begin{align*}
& \text{err}(\hat{\bm \beta}_{\mathcal{A}} ; \bm \beta^*_{\mathcal{A}}) \\
=\ & \text{err}(\bm \beta^*_{\mathcal{A}}; \bm \beta^*_{\mathcal{A}}) 
+ (\hat{\bm \beta}_{\mathcal{A}} - \bm \beta^*_{\mathcal{A}})^\top \text{err}^{'}(\bm \beta^*_{\mathcal{A}}; \bm \beta^*_{\mathcal{A}})  + \frac{(\hat{\bm \beta}_{\mathcal{A}} - \bm \beta^*_{\mathcal{A}})^\top \text{err}^{''}(\bm \beta^*_{\mathcal{A}}; \bm \beta^*_{\mathcal{A}}) (\hat{\bm \beta}_{\mathcal{A}} - \bm \beta^*_{\mathcal{A}}) }{2} + o(\|\hat{\bm \beta}_{\mathcal{A}} - \bm \beta^*_{\mathcal{A}}\|_2^2)
\end{align*}
Taking the expectation of the excess risk and using the fact that \(\mathbb E(\hat{\bm \beta}_{\mathcal{A}} - \bm \beta^*_{\mathcal{A}}) = \bm 0\) and $o(\|\hat{\bm \beta}_{\mathcal{A}} - \bm \beta^*_{\mathcal{A}}\|_2^2) = o(\|\hat{\bm \beta}_{\mathcal{A}} - \bm \beta^*_{\mathcal{A}}\|_2^2) = o(1/n) $, we obtain:
\begin{align*}
\lim_{n \to \infty} n \left[ \mathbb E\left( \text{err}(\hat{\bm \beta}_{\mathcal{A}} ; \bm \beta^*_{\mathcal{A}}) \right) -  \text{err}(\bm \beta^*_{\mathcal{A}} ; \bm \beta^*_{\mathcal{A}}) \right] = \lim_{n \to \infty} n \mathbb E \left[ \frac{(\hat{\bm \beta}_{\mathcal{A}} - \bm \beta^*_{\mathcal{A}})^\top \text{err}^{''}(\bm \beta^*_{\mathcal{A}}; \bm \beta^*_{\mathcal{A}}) (\hat{\bm \beta}_{\mathcal{A}} - \bm \beta^*_{\mathcal{A}}) }{2}   \right].    
\end{align*}

The covariance of the estimator \(\hat{\bm \beta}_{\mathcal{A}}\) is asymptotically given by:
\[
\text{Cov}(\hat{\bm \beta}_{\mathcal{A}} - \bm \beta^*_{\mathcal{A}} )  =  \frac{1}{n} I^{-1}_{\mathcal{A}}.
\]
Since $\text{err}^{''}(\bm \beta^*_{\mathcal{A}}; \bm \beta^*_{\mathcal{A}})$ must be a symmetric matrix,  the expectation of the quadratic form can be expressed as:
\begin{align*}
& \mathbb E \left[ \frac{(\hat{\bm \beta}_{\mathcal{A}} - \bm \beta^*_{\mathcal{A}})^\top \text{err}^{''}(\bm \beta^*_{\mathcal{A}}; \bm \beta^*_{\mathcal{A}}) (\hat{\bm \beta}_{\mathcal{A}} - \bm \beta^*_{\mathcal{A}}) }{2} \right] \\
= \ & \operatorname{trace}\left( \text{err}^{''}(\bm \beta^*_{\mathcal{A}}; \bm \beta^*_{\mathcal{A}}) \cdot \text{Cov}(\hat{\bm \beta}_{\mathcal{A}} - \bm \beta^*_{\mathcal{A}} ) \right) + \mathbb E (\hat{\bm \beta}_{\mathcal{A}} - \bm \beta^*_{\mathcal{A}})^\top \text{err}^{''}(\bm \beta^*_{\mathcal{A}}; \bm \beta^*_{\mathcal{A}})\mathbb E (\hat{\bm \beta}_{\mathcal{A}} - \bm \beta^*_{\mathcal{A}}).
\end{align*}
Substituting the covariance and multiplying by \(n\), we have:
\begin{align*}
\lim_{n \to \infty} n \mathbb E \left[ \frac{(\hat{\bm \beta}_{\mathcal{A}} - \bm \beta^*_{\mathcal{A}})^\top \text{err}^{''}(\bm \beta^*_{\mathcal{A}}; \bm \beta^*_{\mathcal{A}}) (\hat{\bm \beta}_{\mathcal{A}} - \bm \beta^*_{\mathcal{A}}) }{2} \right] = \frac{1}{2} \operatorname{trace}\left(\text{err}^{''}(\bm \beta^*_{\mathcal{A}}; \bm \beta^*_{\mathcal{A}})  \cdot I^{-1}_{\mathcal{A}} \right).
\end{align*}
Defining \(J(\beta) = \frac{1}{2} \left[\operatorname{err}^{''}( \bm {\beta}^{'}; \bm \beta^*) \right]_{\bm {\beta}^{'}  = \bm \beta^*}\), we obtain the final result:
\[
\lim_{n \to \infty} n \left[ \mathbb E\left( \text{err}(\hat{\bm \beta}_{\mathcal{A}} ; \bm \beta^*_{\mathcal{A}}) 
\right) -  \text{err}(\bm \beta^*_{\mathcal{A}}; \bm \beta^*_{\mathcal{A}}) \right] =  \operatorname{trace}\{ J(\bm \beta^*_{\mathcal{A}}) I^{-1}_{\mathcal{A}} \}.
\]
By implementing the same steps on $ \lim_{n \to \infty} n \left[ \mathbb E\left( \text{err}(\tilde{\bm \beta}_{\mathcal{A}} ; \bm \beta^*_{\mathcal{A}}) 
\right) -  \text{err}(\bm \beta^*_{\mathcal{A}} ; \bm \beta^*_{\mathcal{A}}) \right]$, it can be concluded that 
\begin{align*}
\lim_{n \to \infty} n \left[ \mathbb E\left( \text{err}(\tilde{\bm \beta}_{\mathcal{A}} ; \bm \beta^*_{\mathcal{A}})  \right) -  \text{err}(\bm \beta^*_{\mathcal{A}} ; \bm \beta^*_{\mathcal{A}}) \right] = \frac{m+\alpha_0 }{m(1+\alpha_0)} \operatorname{trace}\{ J(\bm \beta^*_{\mathcal{A}}) I^{-1}_{\mathcal{A}} \}.    
\end{align*}
This completes the proof of the lemma.
\end{proof}

From Lemma \ref{lem1}, the ARE can then be expressed as
\begin{align*}
\text{ARE}&=\lim_{n\rightarrow\infty}\frac{\mathbb{E}\{\text{err}(\hat{\bm \beta}; \bm \beta^*)\}-\text{err}(\bm \beta^*; \bm \beta^*)}{\mathbb{E}\{\text{err}(\tilde{\bm \beta}; \bm \beta^*)\}-\text{err}(\bm \beta^*; \bm \beta^*)} \\
&= \lim_{n\rightarrow\infty} \frac{n \left[ \mathbb E\left( \text{err}(\hat{\bm \beta}_{\mathcal{A}} ; \bm \beta^*_{\mathcal{A}}) 
\right) -  \text{err}(\bm \beta^*_{\mathcal{A}}; \bm \beta^*_{\mathcal{A}}) \right]}{ n \left[ \mathbb E\left( \text{err}(\tilde{\bm \beta}_{\mathcal{A}} ; \bm \beta^*_{\mathcal{A}}) \right) -  \text{err}(\bm \beta^*_{\mathcal{A}} ; \bm \beta^*_{\mathcal{A}}) \right]} \\
&= \frac{m(1+\alpha_0)}{m+\alpha_0}. 
\end{align*}
Thus, we have completed  the proof of  Theorem 2.

\section{Proof of  Throrem 3}\label{appC}
\subsection{Basic Lemma and Four Matrices}
The following introduces the first-order optimality conditions for convex optimization, which plays a key role in subsequent proofs.
\begin{lem}\label{lem2}
(Lemma 2.1 in \cite{He2023ExtensionOA}). Let \( \mathbb{X} \subset \mathbb{R}^l \) be a closed convex set, and let \( \theta : \mathbb{R}^l \to \mathbb{R} \) and \( f : \mathbb{R}^l \to \mathbb{R} \) be convex functions. Suppose \( f \) is differentiable on an open set containing \( \mathbb{X} \), and the minimization problem
\[
\min \{ \theta(x) + f(x) \mid x \in \mathbb{Z} \}
\]
has a nonempty solution set. Then, \( x^* \in \arg \min \{ \theta(x) + f(x) \mid x \in \mathbb{X} \} \) if and only if $\forall x \in \mathbb{X}$,
\[
x^* \in \mathbb{X} \ \text{and} \ \theta(x) - \theta(x^*) + (x - x^*)^\top \nabla f(x^*) \geq 0.
\]
\end{lem}

Next, four matrices will be introduced, which will be used in the subsequent convergence analysis.
\begin{align}\label{m1}
\bm Q = \begin{pmatrix}
 \bm E_1  & \bm 0  & \bm 0 \\
\bm 0 & \mu \bm I_n  + \bm E_2  &  \bm 0 \\
\bm 0 &  \bm I_n & \frac{1}{\mu} \bm I_n
\end{pmatrix}, \ \ 
\bm M = \begin{pmatrix}
\bm I_p   & \bm 0  & \bm 0 \\
\bm 0 &  \bm I_n   &  \bm 0 \\
\bm 0 &  \mu \bm I_n &  \bm I_n
\end{pmatrix}, \tag{S1.11}
\end{align}

\begin{align}\label{m2}
\bm H = \begin{pmatrix}
 \bm E_1  & \bm 0  & \bm 0 \\
\bm 0 & \mu \bm I_n  + \bm E_2  &  \bm 0 \\
\bm 0 &  \bm 0 & \frac{1}{\mu} \bm I_n
\end{pmatrix}, \ \
\bm G = \begin{pmatrix}
\bm E_1  & \bm 0  & \bm 0 \\
\bm 0 &  \bm E_2   &  \bm 0 \\
\bm 0 &  \bm 0 &  \frac{1}{\mu} \bm I_n
\end{pmatrix}. \tag{S1.12}
\end{align}
Not difficult to verify, 
\begin{align}\label{m3}
    \bm{HM}= \bm Q \ \text{and} \ \bm G= \bm Q^\top + \bm Q- \bm M^\top \bm H \bm M. \tag{S1.13}
\end{align}

\subsection{Variational inequality characterization}
The constrained optimization form of ALASSO-PLR with manual labels is defined as,
\begin{equation}\label{dsp}
\tag{S1.14}
\begin{aligned}
\min_{\bm{\beta},\bm{r}} \ & L(\bm r) +\lambda \| 
 \bm w \odot \bm \beta \|_1,\\
\text{s.t.} \ \ & \bm X \bm \beta = \bm r, 
\end{aligned}
\end{equation}
where $L(\bm r) = \sum_{i=1}^{n} \left[ -S_i r_i +  m \log\left(1  + \exp({r_i})\right)  \right]$.
And the Lagrange multiplier form of (\ref{dsp}) is   
\begin{align}\label{dsp2}
L(\bm \beta, \bm r;  \bm u) = 
 L(\bm r) +\lambda \| 
 \bm w \odot \bm \beta \|_1  - \bm u^\top (  \bm X \bm \beta - \bm r ). \tag{S1.15}
  \end{align}
According to the variational inequality characterization in section 2 of \cite{He2012OnTO}, we know that the solution of the constrained optimization function above is the saddle point of the  following Lagrangian function in (\ref{dsp2}). 
 As described in \cite{He2012OnTO}, finding a saddle point of \( L(\bm \beta, \bm r;  \bm u)  \) is equivalent to finding \( \bm \beta^*, \bm r^*,  \bm u^* \) such that the following inequalities are satisfied:
\begin{align}\label{vi} 
\bm v^* \in  \Omega, \ L(\bm r) +\lambda \|\bm w \odot \bm \beta \|_1 - L(\bm r^*) -\lambda \| \bm w \odot \bm \beta^* \|_1 + (\bm v -  \bm v^*)^\top F (\bm v^*) \geq 0,  \ \forall \bm v \in \Omega,  \tag{S1.16}
\end{align}
where  $\bm v  =  (\bm \beta^\top, \bm r^\top, \bm u^\top)^\top$,  $\Omega =  \mathbb{R}^{p} \times \mathbb{R}^{n} \times \mathbb{R}^{n}$,
and
\begin{equation}\label{F}\tag{S1.17}
F(\bm v) = \begin{pmatrix}
-\bm X^\top \bm u \\
 \bm u \\
 \bm X \bm \beta - \bm r
\end{pmatrix}.
\end{equation}
Note that the operator \( F \) defined in (\ref{F}) is monotone, because
\begin{equation}\label{f}\tag{S1.18}
(\bm v_1 - \bm v_2)^\top \left[F(\bm v_1) - F(\bm{v}_2)\right] \equiv 0, \quad \forall \ \bm v_1, \bm{v}_2 \in \Omega. 
\end{equation}
Throughout, we denote by \( \Omega^* \) the solution set of (\ref{vi}), which is also the set of saddle points of the Lagrangian function (\ref{dsp2}) of the model (\ref{dsp}).

\subsection{Proof}
 First, we demonstrate that Algorithm 1 can be transformed into a \textbf{prediction-correction} algorithm within the unified framework presented in Section 4.3 of  \cite{He2015PLC}. To this end, We first define,
\[
\check{\bm v}^k = 
\begin{pmatrix}
\check{\bm \beta}^k \\
\check{\bm r}^k \\
\check{\bm u}^k
\end{pmatrix} =
\begin{pmatrix}
{\bm \beta}^{k+1} \\
{\bm r}^{k+1} \\
\bm u^k - \mu ( \bm X \bm \beta^{k+1} - \bm r^k )
\end{pmatrix}.
\]
Here,  we name $ \check{\bm v}^k$ as the prediction variable, where ${\bm \beta}^{k+1}$ and ${\bm r}^{k+1}$ are generated by the $(k+1)$-th iterative sequence of Algorithm 1. Based on this, we can write the \textbf{prediction} as 
\begin{equation}\label{proofadmm}
\tag{S1.19}
\left\{ \begin{aligned}
\check{\bm \beta}^{k}  & \leftarrow  \mathop {\arg \min }\limits_{\bm \beta} \left\{ \lambda \| \bm w \odot \bm \beta \|_1 + \frac{\mu}{2}  \|  \bm X \bm \beta - \bm r^k  - \bm u^k/\mu \|_2^2 + \frac{1}{2} \| \bm \beta - \bm \beta^k   \|^2_{{\bm E}_{1}} \right \};\\ 
\check{\bm r}^{k} & \leftarrow  \mathop {\arg \min }\limits_{\bm r} \left\{ L (\bm r) + \frac{\mu}{2}  \|  \bm X \bm \check{\bm \beta}^{k} - \bm r  - \bm u^k/\mu \|_2^2 + \frac{1}{2} \| \bm r - \bm r^k   \|^2_{\bm E_{2}} \right \}; \\
\check{\bm u}^{k} & \leftarrow  \bm{u}^{k} - \mu(\bm X \check{\bm \beta}^{k} - {\bm r}^{k} ), 
\end{aligned} \right.
\end{equation}
The above iterative steps can be written in a form similar to \eqref{vi}, and the specific cases can be summarized in the following lemma.
 \begin{lem}\label{lem3}
Given $\bm v^k$, for any $\bm v \in \Omega$, the predicted variable $\check{\bm v}^k$ generated by \eqref{proofadmm} satisfies
\begin{align}\notag
\ L(\bm r) +\lambda \| \bm w \odot \bm \beta \|_1 - L(\check{\bm r}^k) -\lambda \| 
 \bm w \odot \check{\bm \beta}^k \|_1 + (\bm v -  \check{\bm v}^k)^\top F (\check{\bm v}^k) \geq (\bm v -  \check{\bm v}^k)^\top \bm Q ({\bm v}^k - \check{\bm v}^k),   
\end{align}
where $\bm Q$ is defined in \eqref{m1}.
\end{lem}
\begin{proof}
According to Lemma \ref{lem2} of the first-order optimality principle, the optimality condition for the subproblem of $\check{\bm \beta}^k$ in \eqref{proofadmm} is
\begin{align*}
\lambda \| \bm w \odot \bm \beta \|_1 - \lambda \| \bm w \odot \check{\bm \beta}^k \|_1 + (\bm \beta -  \check{\bm \beta}^k)^\top \left[ \mu \bm X^\top (\bm X \check{\bm \beta}^k - \bm r^k  - \bm u^k/\mu) + \bm E_1(\check{\bm \beta}^k -  {\bm \beta}^k)  \right] \ge 0.
\end{align*}
Using $\check{\bm u}^{k}  = \bm{u}^{k} - \mu(\bm X \check{\bm \beta}^{k} - {\bm r}^{k} )$ in \eqref{proofadmm}, the above equation is 
\begin{align}\label{betaproof}
\lambda \|\bm w \odot \bm \beta \|_1 - \lambda \|\bm w \odot \check{\bm \beta}^k \|_1
+ (\bm \beta -  \check{\bm \beta}^k)^\top (-\bm X^\top \check{\bm u}^k + \bm E_1(\check{\bm \beta}^k -  {\bm \beta}^k)  ) \ge 0. \tag{S1.20}
\end{align}    
Similarly, the optimality condition for the $\check{\bm r}^{k}$-subproblem can be written as
\begin{align}\label{rproof}
L (\bm r) - L (\check{\bm r}^{k}) + (\bm r - \check{\bm r}^{k})^\top\left[ \check{\bm u}^{k} + \mu( \check{\bm r}^{k} - \bm r^k ) +  \bm E_2 ( \check{\bm r}^{k} - \bm r^k )  \right] \ge 0. \tag{S1.21}
\end{align}
The formula for generating  $\check{\bm u}^{k}$ in \eqref{proofadmm}   can also be written as
\begin{align*}
    (\bm X \check{\bm \beta}^{k} - \check{\bm r}^{k} ) + (\check{\bm r}^{k} - {\bm r}^{k}) + \frac{1}{\mu}(\check{\bm u}^{k} - {\bm u}^{k}) = 0.
\end{align*}
Then, we have 
\begin{align}\label{uproof}
( \bm u - \check{\bm u}^{k})^\top  \left[ (\bm X \check{\bm \beta}^{k} - \check{\bm r}^{k} ) + (\check{\bm r}^{k} - {\bm r}^{k}) + \frac{1}{\mu}(\check{\bm u}^{k} - {\bm u}^{k})  \right] \ge 0. \tag{S1.22}
\end{align}
Combining \eqref{betaproof}, \eqref{rproof}, and \eqref{uproof}  together can lead to the conclusion of Lemma \ref{lem3}.\end{proof}

Note that the iteration steps of \eqref{proofadmm} are the same as the iteration variables $\bm \beta$ and $\bm r$ generated by our proposed Algorithm 1, but the iteration of $\bm u$ is different. Therefore, we need the following \textbf{correction} step, 
\begin{align}\label{corr}
  \bm v^{k+1} =  \bm v^{k} - \bm M (\bm v^{k} - \check{\bm v}^{k}), \tag{S1.23}
\end{align}
where $\bm M$ is defined in \eqref{m1}.  The above discussion shows that the proposed Algorithm 1 can be divided into two steps: prediction in \eqref{proofadmm}   and correction in \eqref{corr}. This prediction and correction step has been used as an effective tool to demonstrate the convergence of the ADMM algorithm, as detailed in \cite{He2023ExtensionOA}, \cite{He2015PLC}  and its references.

As shown in Section 4.3 of \cite{He2015PLC}, this prediction-correction step requires the following condition to ensure algorithm convergence, that is, $$\bm H = \bm Q \bm M^{-1} \succ \bm 0 \ \text{and} \ \bm G = \bm Q^\top + \bm Q- \bm M^\top \bm H \bm M \succ \bm 0.   $$
It can be easily inferred from \eqref{m2} and \eqref{m3} that Algorithm 1 satisfies the above convergence conditions. Then, from the results of Theorems 5.2  and 5.6 of \cite{He2015PLC}, we have 
\begin{align}
    \|\bm v^{k+1} - \bm v^{*} \|_{\bm{H}}^2  \le \|\bm v^{k} - \bm v^{*} \|_{\bm{H}}^2 - \|\bm v^{k} - \check{\bm v}^{k} \|_{\bm{G}}^2, \label{result1} \tag{S1.24} \\
    \|\bm M (\bm v^{k+1} - \check{\bm v}^{k+1}) \|_{\bm{H}}^2 \le \|\bm M (\bm v^{k} - \check{\bm v}^{k}) \|_{\bm{H}}^2. \label{result2} \tag{S1.25}
\end{align}

The inequality in \eqref{result1} implies that the sequence \(\{\bm v^k\}\) exhibits Fejér monotonicity. When \(\bm G \succ 0\), this sequence converges to a point \(\bm v^* \in \Omega^*\) in the \(\bm H\)-norm. This key inequality is fundamental for establishing the global convergence of the sequence. The proof of the global convergence can be found in Theorem 3.2 of \cite{He2023ExtensionOA}.  So far, the proof of the first property of Theorem 3 has been completed.

Since  $\bm M (\bm v^{k+1} - \check{\bm v}^{k+1}) = \bm v^{k+1} - \bm v^{k+2}$ and $\bm M (\bm v^{k} - \check{\bm v}^{k}) = \bm v^{k} - \bm v^{k+1}$, the inequality in \eqref{result2} implies that  \begin{align}\label{result22} 
    \| \bm v^{k+1} - \bm v^{k+2} \|_{\bm{H}}^2 \le \| \bm v^{k} - \bm v^{k+1} \|_{\bm{H}}^2. \tag{S1.26}
\end{align}
From the inequality in \eqref{result1}, it follows from $\bm G \succ 0$, there is a constant $c_0 > 0$ such that
\begin{align}\label{result3} 
\|\bm v^{k+1} - \bm v^{*} \|_{\bm{H}}^2 \le \|\bm v^{k} - \bm v^{*} \|_{\bm{H}}^2 - c_0\|\bm M (\bm v^{k} - \check{\bm v}^{k}) \|_{\bm{H}}^2 = \|\bm v^{k} - \bm v^{*} \|_{\bm{H}}^2 - c_0\| \bm v^{k} - \bm v^{k+1} \|_{\bm{H}}^2. \tag{S1.27}
\end{align}

By summing up inequality in \eqref{result3} over \(k = 0, 1, \cdots, K\), we obtain 
\begin{align*}
    c_0 \sum_{k=0}^K \| \bm v^{k} - \bm v^{k+1} \|_{\bm{H}}^2 \le \|\bm v^{0} - \bm v^{*} \|_{\bm{H}}^2.
\end{align*}
Combining \eqref{result22} leads to the conclusion that 
\begin{align}
    c_0 (K+1)  \| \bm v^{K} - \bm v^{K+1} \|_{\bm{H}}^2 \le \|\bm v^{0} - \bm v^{*} \|_{\bm{H}}^2. \tag{S1.28}
\end{align}
Therefore, we have arrived at the second conclusion of Theorem 3.

\section{Proof of  Throrem 4}\label{appD}
Recall that we denote $\left\{ \hat{\bm \beta}^k, \hat{\bm z}^k, \hat{\bm u}^k\right\}$ as the results of the $k$-th iteration of Algorithm 1, and $\left\{ \bar{\bm \beta}^k, \bar{\bm z}^k, \bar{\bm u}^k \right\}$ as the results of the $k$-th iteration of Algorithm 2. We make the assumptions that two algorithms use the same $\eta$ and
\begin{align}\label{ass}
\left\{ \hat{\bm \beta}^k, \hat{\bm z}^k, \hat{\bm u}^k\right\} = \left\{ \bar{\bm \beta}^k, \bar{\bm z}^k, \bar{\bm u}^t \right\}. \tag{S1.29}
\end{align}
Then, we will conduct an analysis of the iteration scenario at step $k + 1$. 

Firstly, we will discuss the update of the $\bm \beta$-subproblem. For Algorithm 1, we have 
\begin{align}\label{dbeta}
\hat{\bm \beta}^{k+1} \leftarrow \arg \min_{\boldsymbol\beta} \left\{ \lambda \| 
 \bm w \odot \bm \beta \|_1 + \frac{\eta}{2} \| \bm \beta - \hat{\bm \beta}^k + \frac{\mu \bm X^\top(\bm X \hat{\bm \beta}^k - \hat{\bm r}^k  - \hat{\bm u}^k/\mu) }{\eta}   \|_{2}^2 \right\}. \tag{S1.30}
\end{align}    
and for Algorithm 2, we have
\begin{equation}
\tag{S1.31}
\begin{aligned}\label{dbeta2}
\bar{\bm{\beta}}^{k + 1} \leftarrow \underset{\boldsymbol{\beta}}{\mathrm{argmin}} \left\{ \lambda \left\lVert \bm{w} \odot \bm{\beta} \right\rVert_1 + \frac{\eta}{2} \left\lVert \bm{\beta} - \bar{\bm{\beta}}^k + \sum_{g = 1}^{G} \bm{\xi}_g^k \right\rVert_2^2 \right\}, 
\end{aligned}    
\end{equation}
where $\bm{\xi}_g^k = \frac{\mu \bm{X}_g^\top(\bm{X}_g \bar{\bm{\beta}}^k - \bar{\bm{r}}_g^k - \bar{\bm{u}}_g^k / \mu)}{\eta}$.
Since $\bm X = [\bm X_1^\top, \bm X_2^\top,\dots,\bm X_G^\top ]^\top$, $\bar{\bm{r}}^k = [(\bar{\bm{r}}_1^k)^\top,(\bar{\bm{r}}_2^k)^\top,\dots,(\bar{\bm{r}}_G^k)^\top]^\top$ and $\bar{\bm{u}}^k = [(\bar{\bm{u}}_1^k)^\top,(\bar{\bm{u}}_2^k)^\top,\dots,(\bar{\bm{u}}_G^k)^\top]^\top$,  it can be concluded that 
\begin{align}\label{prc3}
\sum_{g = 1}^{G} \bm{\xi}_g^k =  \frac{\mu \bm X^\top(\bm X \bar{\bm \beta}^k - \bar{\bm r}^k  - \bar{\bm u}^k/\mu) }{\eta}. \tag{S1.32}
\end{align}
According to the assumption in \eqref{ass}, there is  $\hat{\bm \beta}^k =  \bar{\bm \beta}^k $, $\hat{\bm r}^k = \bar{\bm r}^k$ and  $\hat{\bm u}^k = \bar{\bm u}^k$. It follows from \eqref{dbeta} and  \eqref{dbeta2} that
\begin{align}\label{prbeta}
 \hat{\bm \beta}^{k+1} = \bar{\bm \beta}^{k+1}. \tag{S1.33}
\end{align}

Next, we will discuss the update of the $\bm r$-subproblem. For Algorithm 1, we have  
\begin{equation}
\tag{S1.34}
\begin{aligned}\label{dr}
  \hat{\bm r}^{k+1} \leftarrow \left[ \Phi( \hat{\bm r}^k) + \mu \bm I_p   \right]^{-1}  \left[ \mu  ( \bm X  \hat{\bm \beta}^{k+1} -  \hat{\bm u}^k/\mu) + \Phi( \hat{\bm r}^k)  \hat{\bm r}^k  -  \phi^{'}( \hat{\bm r}^k) + \bm S \right].
\end{aligned}
\end{equation}
$\hat{\bm r}^{k + 1}$ can be expressed as $[(\hat{\bm r}_1^{k + 1})^\top, (\hat{\bm r}_2^{k + 1})^\top, \ldots, (\hat{\bm r}_G^{k + 1})^\top]$, and $\bm S$, $\hat{\bm u}^k$ and $\hat{\bm r}^k$ are divided in a similar manner. Drawing on the definition of $\Phi( \hat{\bm r}^k)$, it becomes clear that the diagonal elements of its block-diagonal matrix are $\Phi( \hat{\bm r}_g^k)$, with $g$ ranging from $1$ to $G$.  Then, we have  
\begin{equation}
\tag{S1.35}
\begin{aligned}\label{dr2}
  \hat{\bm r}_g^{k+1} \leftarrow \left[ \Phi( \hat{\bm r}_g^k) + \mu \bm I_{n_g}   \right]^{-1}  \left[ \mu  ( \bm X  \hat{\bm \beta}_g^{k+1} -  \hat{\bm u}_g^k/\mu  ) 
  + \Phi( \hat{\bm r}_g^k)  \hat{\bm r}_g^k  -  \phi^{'}( \hat{\bm r}_g^k) + \bm S_g \right], g=1,\dots,G.
\end{aligned}
\end{equation}
For Algorithm 2, we get
\begin{align}\label{dr3} 
  \bar{\bm r}_g^{k+1} \leftarrow \left[ \Phi(\bar{\bm r}_g^k) + \mu \bm I_{n_g}   \right]^{-1}  \left[ \mu  ( \bm X_g \bar{\bm \beta}^{k+1} - \bar{\bm u}_g^k/\mu  ) + \Phi(\bar{\bm r}_g^k) \bar{\bm r}_g^k  -  \phi^{'}(\bar{\bm r}_g^k) + \bm S_g \right], g=1,\dots,G. \tag{S1.36}
\end{align}
Again,  according to the assumption in \eqref{ass} and \eqref{prbeta},   there is  $\hat{\bm r}_g^k = \bar{\bm r}_g^k$,  $\hat{\bm u}_g^k = \bar{\bm u}_g^k$ and $\hat{\bm \beta}^{k+1} = \bar{\bm \beta}^{k+1}$. 
It follows from   \eqref{dr2} and  \eqref{dr3} that
\begin{align}\label{prz3}
 \hat{\bm r}_g^{k+1} = \bar{\bm r}_g^{k+1}, g =1,\dots,G. \tag{S1.37}
\end{align}
Thus, we obtain 
\begin{align}\label{dr4}
    \hat{\bm r}^{k+1} = \bar{\bm r}^{k+1}. \tag{S1.38}
\end{align}
Finally, we will discuss the update of the $\bm u$-subproblem.
For Algorithm 1,
$$\hat{\bm u}^{k + 1} = \hat{\bm{u}}^{k} - \mu(\bm X \hat{\bm \beta}^{k + 1} - \hat{\bm r}^{k + 1} ).$$
Hence,
\begin{align}
  \hat{\bm u}_g^{k + 1} = \hat{\bm{u}}_g^{k} - \mu(\bm X_g \hat{\bm \beta}^{k + 1} - \hat{\bm r}_g^{k + 1} ).  \tag{S1.39}
\end{align} 
For Algorithm 2,
$$\bar{\bm u}_g^{k + 1} = \bar{\bm{u}}_g^{k} - \mu(\bm X_g \bar{\bm \beta}^{k + 1} - \bar{\bm r}_g^{k + 1} ).$$
It follows from \eqref{ass}, \eqref{prbeta} and \eqref{prz3} that 
\begin{align}\label{pru0}
 \hat{\bm u}_g^{t+1} = \bar{\bm u}_g^{t+1}, g =1,\dots,G. \tag{S1.40}
\end{align}
Thus, we obtain 
\begin{align}\label{pru}
    \hat{\bm u}^{k+1} = \bar{\bm u}^{k+1}. \tag{S1.41}
\end{align}
Therefore, we can conclude that if the iteration sequences of the two algorithms (with the same $\eta$) at step $k$ are identical, those at step $k + 1$ will be too.  Based on this, assuming that \(\{ \hat{\bm \beta}^0, \hat{\bm r}^0, \hat{\bm u}^0 \} = \{ \bar{\bm \beta}^0, \bar{\bm r}^0 , \bar{\bm u}^0 \}\), it follows that \(\{ \hat{\bm \beta}^1, \hat{\bm r}^1, \hat{\bm u}^1 \} = \{ \bar{\bm \beta}^1, \bar{\bm r}^1 , \bar{\bm u}^1 \}\). By the same token, we can deduce that 
\begin{align}
\left\{ \hat{\bm \beta}^k, \hat{\bm r}^k, \hat{\bm u}^k \right\} = \left\{ \bar{\bm \beta}^k, \bar{\bm z}^k , \bar{\bm u}^k \right\}, \quad \text{for all } k. \tag{S1.42}
\end{align}
Up to this point, the proof of the theorem has been successfully completed.

\section{Other Real Dataset Experiments}\label{appF}
Here, we add three additional datasets from Section 5.2 regarding numerical experiments on ARE. Since the data scales of the three datasets vary, we partitioned them into training sets and prediction sets respectively as follows: for ``ijcmn1", it's (45,000 + 4,990); for ``a9a", (30,000 + 2,561); and for ``phishing", (10,000 + 1,055). Other settings related to data experiments are the same as in Section 5.2. The subsequent Figure \ref{fig3} is going to illustrate the relationship among the ARE, \(m\), and \(\alpha_0\). In the first sub-figure, with \(\alpha_0\) set at a constant value of 1, the theoretical formula for the ARE is \(2-\frac{2}{m + 1}\). Meanwhile, in the second sub-figure, when \(m\) is held constant at 2, the corresponding theoretical formula for the ARE is \(2 - \frac{2}{\alpha_0+2}\).

The numerical outcomes depicted in the figure reveal that when it comes to the fitting of the theoretical ARE, Dataset “ijcmn1” exhibits the optimal performance. It is succeeded by Dataset “a9a”, and then by Dataset “phishing”. The disparities among these three datasets could potentially be ascribed to the ratio of the sample size to the magnitude of the parameters awaiting estimation. Specifically, Dataset “ijcmn1” boasts the highest such ratio, whereas Dataset “phishing” has the lowest.

\begin{figure}[H]
    \centering
    \includegraphics[width=0.9\linewidth]{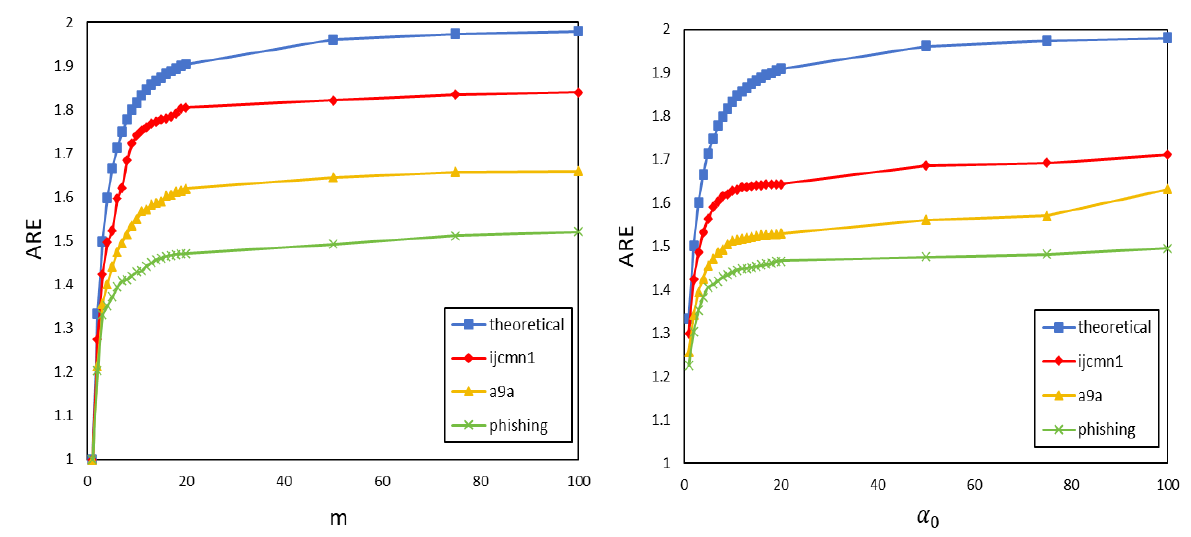}
    \caption{\footnotesize{A schematic diagram showing how the ARE value changes with the variation of the values of \(m\) and \(\alpha_0\)}.}
    \label{fig3}
\end{figure}

\begin{small}
\bibliographystyle{elsarticle-num}
\bibliography{myrefq(abb)}  
\end{small}

\end{document}